\documentclass[10pt]{article}

\usepackage{PRIMEarxiv}

\usepackage{amsmath,amsfonts,bm}

\def\eqref#1{equation~\ref{#1}}

\def\1{\bm{1}}

\def\ra{{\textnormal{a}}}

\def\rt{{\textnormal{t}}}

\DeclareMathAlphabet{\mathsfit}{\encodingdefault}{\sfdefault}{m}{sl}
\SetMathAlphabet{\mathsfit}{bold}{\encodingdefault}{\sfdefault}{bx}{n}

\newcommand{\E}{\mathbb{E}}

\newcommand{\R}{\mathbb{R}}

\DeclareMathOperator*{\argmax}{argmax}

\usepackage[hypertexnames=false]{hyperref}
\usepackage{url}
\usepackage{booktabs}
\usepackage{xcolor}
\usepackage{natbib}
\usepackage{amsmath,amssymb,amsthm,hyperref,relsize,bbm,xcolor,thm-restate,graphicx}

\newcommand{\lt}{\left(}
\renewcommand{\rt}{\right)}
\newcommand{\norm}[1]{\left\lVert#1\right\rVert}
\newcommand{\f}{\dfrac}
\newcommand{\ff}{\tfrac}
\newcommand{\su}[2]{\mathlarger{\sum}\limits_{#1}^{#2}}
\renewcommand{\l}{\ell}
\newcommand{\ep}{\epsilon}
\newcommand{\T}{\theta}
\renewcommand{\L}{\mathcal{L}}
\newcommand{\F}{\mathcal{F}}
\newcommand{\de}{\delta}
\newcommand{\M}{\lambda}
\newcommand{\G}{\mathcal{G}}
\newcommand{\X}{\mathcal{X}}
\newcommand{\D}{\mathcal{D}}
\newcommand{\Y}{\mathcal{Y}}
\renewcommand{\ra}{\rightarrow}

\newcommand{\name}{\emph{\sc{Barack}}}
\usepackage{algorithm}
\usepackage{algorithmicx}
\usepackage{algpseudocode}

\renewcommand{\P}{\mathcal{P}}

\newtheorem{theorem}{Theorem}

\newtheorem{lemma}{Lemma}

\usepackage[capitalize]{cleveref}
\crefname{section}{Section}{Secs.}
\Crefname{section}{Section}{Sections}
\Crefname{table}{Table}{Tables}
\crefname{table}{Tab.}{Tabs.}

\begin{document}

\title{\name: Partially Supervised Group Robustness With Guarantees}

\newcommand*\samethanks[1][\value{footnote}]{\footnotemark[#1]}

\author{
  Nimit Sohoni\thanks{Work done at Meta AI.} \\
  Stanford University \\
  \texttt{nims@stanford.edu} \\
  \And
  Maziar Sanjabi \\
  Meta AI \\
  \texttt{maziars@fb.com} \\
  \And
  Nicolas Ballas$^{\dag}$ \\
  Meta AI \\
  \texttt{ballasn@fb.com} \\
  \And
  Aditya Grover\samethanks$^{\,\,\,,\dag}$ \\
  UCLA \\
  \texttt{adityag@cs.ucla.edu} \\
  \And
  Shaoliang Nie$^{\dag}$ \\
  Meta AI \\
  \texttt{snie@fb.com} \\
  \And 
  Hamed Firooz \\
  Meta AI \\
  \texttt{mhfirooz@fb.com} \\
  \And 
  Christopher R\'e \\
  Stanford University \\
  \texttt{chrismre@cs.stanford.edu} \\
}
\maketitle
\def\thefootnote{$^{\dag}$}\footnotetext{These authors contributed equally to this work.}\def\thefootnote{\arabic{footnote}}

\begin{abstract}
  While neural networks have shown remarkable success on classification tasks in terms of average-case performance, they often fail to perform well on certain groups of the data. Such group information may be expensive to obtain; thus, recent works in robustness and fairness have proposed ways to improve worst-group performance even when group labels are \emph{unavailable} for the training data. However, these methods generally underperform methods that utilize group information at training time.
  In this work, we assume access to a small number of group labels alongside a larger dataset without group labels. We propose \name, a simple two-step framework to utilize this partial group information to improve worst-group performance: train a model to predict the missing group labels for the training data, and then use these predicted group labels in a robust optimization objective.
  Theoretically, we provide generalization bounds for our approach in terms of the worst-group performance, which scale with respect to both the \emph{total} number of training points and the number of training points with \emph{group labels}. 
  Empirically, our method outperforms the baselines that do not use group information, even when only 1-33\% of points have group labels.
We provide ablation studies to support the robustness and extensibility of our framework.
\end{abstract}

\section{Introduction}
On classification tasks, deep neural networks can often underperform on certain groups of the data. For example, in datasets with ``spurious correlations,'' standard neural networks have been shown to achieve high average accuracy, yet drastically lower accuracy on groups that violate the spurious correlation \citep{sagawa2019distributionally}. Similarly, when certain groups are underrepresented in the training data, models tend to perform poorly on these rare groups \citep{sohoni2020no}. In many settings, such as applications where fairness or safety are important, this behavior is undesirable; for example, gender classification systems have been shown to underperform for non-white faces \citep{buolamwini2018gender}, and medical triage systems have been shown to miss certain abnormality subtypes \citep{oakden2020hidden}. To avoid this, we want to ensure \emph{group robustness}, i.e., high accuracy on the worst-performing group.

Unfortunately, group annotations are often unavailable. Many datasets only have labels for the \emph{task}, not the groups. Group labels may also be relatively expensive to obtain; for instance, in the common setting where the group labels are finer-grained than the class labels, it may require a higher annotation cost to obtain group labels than class labels \citep{gebru2017scalable}. Or, there may be privacy concerns with acquiring additional group labels.
This paucity of group labels makes ensuring group robustness more challenging.

Existing works to address the issue of group robustness fall into two main categories: those that assume access to the group labels for \emph{all} of the training data, and those that assume \emph{no} access to the group labels for the training data. For instance, in the first category, \citet{sagawa2019distributionally} propose \emph{group distributionally robust optimization} (GDRO), an efficient algorithm for minimizing the worst-group loss when the groups are known.

More recently, several approaches have been proposed to improve group robustness when group labels are \emph{unavailable}. A common approach is to first estimate the group labels, then train a robust classifier using these estimated group labels \citep{sohoni2020no,liu2021just}.
However, in terms of worst-group performance, the methods that require group labels unsurprisingly (and often substantially) outperform those that do not.

A fundamental question is: \textbf{can we close this gap if we have \emph{partial} group information}? Specifically, we seek to understand the intermediate regime in which group labels are available for some (small) subset of the training data, while the remainder of the data has class labels only.
The distinction between this setting and the aforementioned prior work is akin to the difference between semi-supervised learning vs. supervised or unsupervised learning.
From an application standpoint, when the {identities} of the groups are known, it is often feasible to obtain group labels for a \emph{small} subset of the data.
From a theoretical standpoint, the relative value of class vs. group labels for ensuring group robustness is still unknown.

To address this question, we propose \name,\footnote{Name inspired by GEORGE \citep{sohoni2020no}, a baseline for the setting where no group labels are known.} a simple two-stage approach to improve group robustness for the setting wherein group labels are only known for a subset of datapoints. In the first stage of \name{}, we use the available group labels to train a model to predict the group labels on the datapoints without group annotations. In the second stage, we use these predicted group labels in the GDRO objective \citep{sagawa2019distributionally} to train a robust model.

Theoretically, we show how the \emph{worst-group} generalization performance scales with the number of total points and the number of points with group labels. Empirically, we show that even if only a small fraction (1-33\%) of points have group labels, \name{} improves over approaches that do not use group labels.

\textbf{Contributions}. In summary, our main contributions are:
\begin{itemize}
    \item We propose a simple framework, \name{}, that can improve group robustness with only a small number of group labels: we train a model to predict the missing group labels, then use these group labels in a robust training objective (GDRO).
    \item On four benchmark image classification datasets---MNIST, Waterbirds, CelebA, and CIFAR-10---we show that with as few as 1-33\% of the points having group labels, our method empirically outperforms all baselines that do not use group information, approaching the performance of GDRO trained on the full dataset (Full-GDRO).
    \item We prove a generalization bound on the worst-group performance of our method, showing that it scales with the inverse square root of the total number of points \emph{with group labels} in the smallest group. We show how to tighten this bound using semi-supervised learning (SSL). Under additional assumptions on the structure of the errors made in the first stage, we also show that our method also guarantees an improved generalization bound compared to empirical risk minimization (ERM).
    \item We conduct ablation studies to better understand the importance of the different components of \name. We also show that by using more complex techniques such as SSL, we can improve the worst-group performance of \name{} even further.
\end{itemize}

\begin{figure*}
\includegraphics[width=\textwidth]{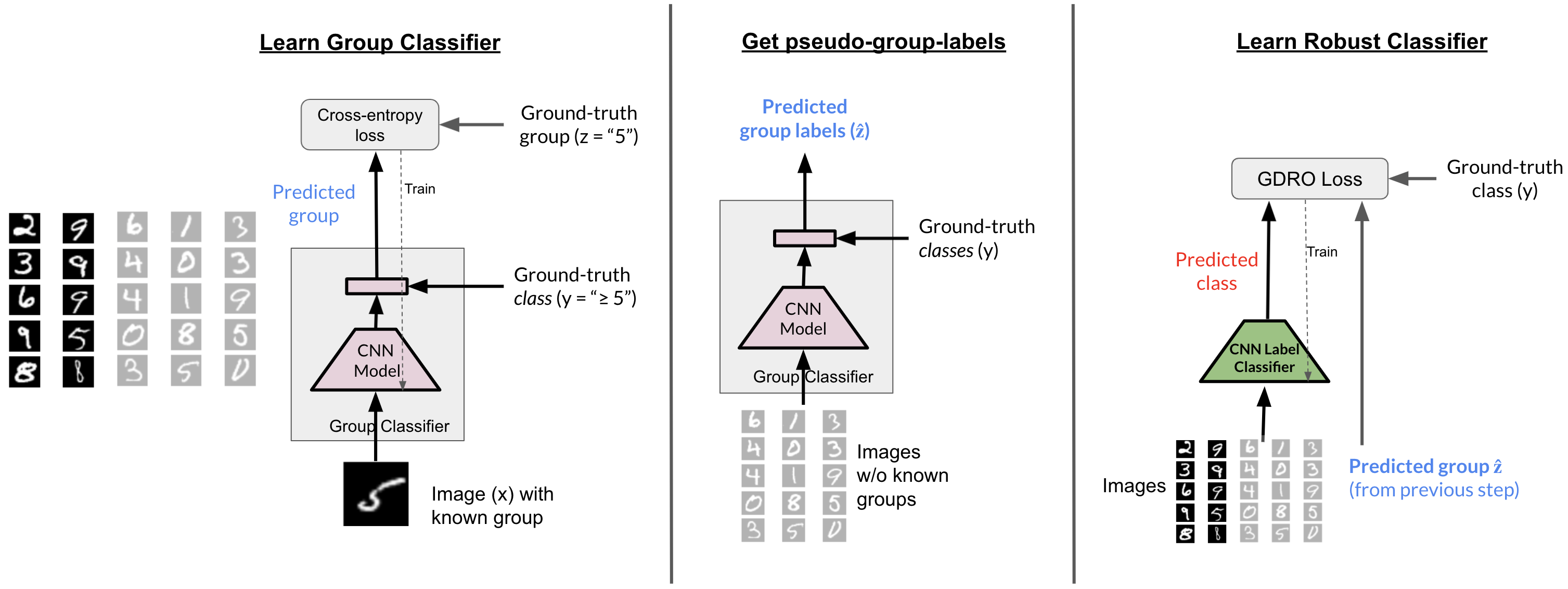}
\caption{Schematic describing \name-Base.  Inputs are images and class labels; some images also have known ground-truth group labels (darker examples).  First, we train a model to classify the groups, using the images with known group labels (left panel). Then, we use this model to compute ``group pseudolabels'' for the remaining training images (middle panel). Finally, we use these group pseudo-labels in the GDRO loss to train a more robust model (right panel).
(Other versions of \name, e.g. \name-SSL, differ in the ``Learn Group Classifier'' stage, where other methods, e.g. semi-supervised learning, can be used.)
}
\label{fig:overall_method}
\end{figure*}

\section{Background}
\subsection{Problem Setup}
We consider a similar setting to that of \citet{sagawa2019distributionally}: we have $n$ training points sampled IID from a distribution $\mathcal{P}$: $\{(x_i,y_i,z_i)\}_{i=1}^n \in \mathcal{X} \times \mathcal{Y} \times \mathcal{G}$. $x_i$ denotes a datapoint, $y_i$ its (discrete) class label, and $z_i$ its (discrete) group label. However, unlike \citet{sagawa2019distributionally}, we do not assume that we know all of the $z_i$'s; rather, we only assume knowledge of $z_1, \dots, z_m$, where $m < n$.
We denote $\D_1 := \{(x_i,y_i,z_i)\}_{i=1}^m$ to be the group-labeled dataset, and $\D_2 := \{(x_i,y_i)\}_{i=m+1}^n$ to be the dataset of group-unlabeled points. Note that we assume $\D_1$ and $\D_2$ are samples from the same distribution; but the $z_i$'s are unobserved on $\D_2$.

Our end goal is to maximize the worst-group accuracy on the task of predicting the correct \emph{class} label for each datapoint. In other words, given a function class of classifiers $\mathcal{F}$ (where each $f \in \mathcal{F}$ is a function $\X \ra \Delta^{|\Y|}$, i.e., a function that outputs probabilities for each class), we wish to find $f \in \mathcal{F}$ that maximizes ${\min\limits_{g \in \mathcal{G}} \E_{(x,y) | z = g} [\1(\argmax \{f(x)\} = y)]}$.

In practice we instead seek the $f \in \mathcal{F}$ that minimizes the worst-group loss over the training data: \begin{equation}\label{eq:gdro}\max\limits_{g \in \mathcal{G}} \E_{(x_i,y_i) | z_i = g} [\ell(f(x_i),  y_i)].\end{equation}

When the $z_i$'s are known, the latter problem can be solved with group DRO (GDRO) \citep{sagawa2019distributionally}. GDRO is a stochastic optimization method designed for minimax problems of exactly the form of \eqref{eq:gdro}. However, in our setting, solving this problem is challenging because we only know a subset of the $z_i$'s, so we cannot compute \eqref{eq:gdro} directly.

\subsection{Related Work}
\paragraph{Beyond average performance.}
Our work primarily builds on prior work in the area of group robustness. This line of work has a long history in the literature; for example, see~\citep{mohri2019agnostic, zhang2020coping}  and references therein.
Several methods have been proposed to improve group robustness when the group labels are known at train time.
In our algorithms, we focused on GDRO~\citep{sagawa2019distributionally}, a stochastic algorithm for minimizing the worst-group loss; we use GDRO as a component of our method.
While we mainly focus on worst-group accuracy, average and worst-group accuracy are not the only measures of performance of interest to ML practitioners. Several other works have used different approaches to strike a balance between important performance measures, for example through the lenses of distributional robustness~\citep{duchi2019distributionally, wang2020robust, zhang2020coping, ben2013robust}, fairness~\citep{hardt2016equality, agarwal2018reductions, li2019fair, li2020tilted}, or outlier/noisy sample detection~\citep{huber1992robust, bhatia2015robust, menon2019can, li2020tilted}. 
\citet{balashankar2019what} and \citet{martinez2020minimax} propose different methods for ensuring group \emph{Pareto fairness}, which seeks to find Pareto-efficient solutions in terms of the accuracies on each group (a more general problem than GDRO).\footnote{In fact, we note that our proposed algorithm \name~can be viewed as an instantiation of the plug-in estimator proposed in Theorem 4.2 of \citet{martinez2020minimax}.}

\paragraph{Group robustness without group labels.}
When the group labels are not known, alternative methods exist that still attempt to improve group robustness. Several of these works aim to first estimate the group labels, then train a robust classifier using these estimated group labels \citep{sohoni2020no,nam2020learning,liu2021just,zhang2022correctncontrast}. Others make no assumptions on the structure of the groups, and simply try to perform well on ``all possible'' data subsets above a specified size \citep{levy2020large,martinez2021blind}. Unsurprisingly, these approaches typically underperform methods that do utilize group labels.

Our method, \name{}, involves training two models sequentially. This is commonly used in different ways as an approach to increasing model robustness in the literature~\citep{yaghoobzadeh2019increasing, utama2020towards}. Among these works, \citep{liu2021just, goel2020model, creager2021environment, nam2020learning, sohoni2020no,zhang2022correctncontrast} are most relevant to our work, where a model is trained first and then the outputs of this model are used in some manner (such as in the GDRO objective) to train the second model to be robust. Our key point of difference is that none of these works are designed to actually utilize possible group labels when they are known for some samples.
\name{}~can yield superior performance to these methods by utilizing such additional group information, even if it is limited. 

\paragraph{Semi-supervised learning.}
Semi-supervised learning (SSL) is a rich field with several recent developments. For \name-SSL we use FixMatch \citep{sohn2020fixmatch}, a recent state-of-the-art method for SSL (which performs especially well on CIFAR-10). Other recent successful approaches to semi-supervised learning involve learning \emph{self-supervised} representations (without using labels), and then using the labeled examples for fine-tuning; examples include \citep{xie2019unsupervised,simclr,swav}.

Our work also has connections to self-training. In standard self-training, a labeled dataset is used to train a model to generate pseudolabels for a separate unlabeled dataset; the labeled and pseudolabeled data are then used together to train a downstream model \citep{zoph2020rethinking, xie2020self, lee2013pseudo, rosenberg2005semi}. In our work, we instead generate pseudolabels for the task of classifying the \emph{groups}, which are then used to train a robust model for the original task. 

\paragraph{Concurrent work.}
We would also like to acknowledge the following important concurrent work: Spread Spurious Attribute (SSA) \citep{iclr_2022_SSA}, which considers a very similar problem to ours: specifically, they consider the problem of group robustness when there are spurious attributes that are known for a subset of the training data. Their proposed algorithm can be viewed as a special case of our general two-step framework (\name), in which they use semi-supervised learning techniques for the group prediction stage. Compared to \citep{iclr_2022_SSA}, our work is more focused on understanding and analyzing the effectiveness of \name{} through theoretical and ablation analyses. We believe that these analyses provide valuable insight regardless of the precise method used to estimate group labels. As we show in our experiments and analysis, while the basic version of \name{} performs quite well, better training methods for the first (group classification) stage, such as with SSL, can translate to better results (possibly at the cost of increased computational complexity).

\citep{lokhande2022towards} also address a similar version of our partial group robustness problem, using a different approach based on minimizing an upper bound to the GDRO loss. However, unlike our work (and that of \citep{iclr_2022_SSA}), they avoid any estimation of missing group labels due to privacy considerations, which unfortunately results in substantially lower worst-group performance (albeit still better than ERM) due to their upper bound function possibly being quite loose.

\section{Method}
\label{sec:method}

\begin{figure}
\begin{algorithm}[H]
\caption{\name (General)}
\label{alg:general}
\begin{algorithmic}[1]
\Require Group-labeled data $\mathcal{D}_1 = \{(x_i,y_i,z_i)\}_{i=1}^m$, group-unlabeled data $\mathcal{D}_2 = \{(x_i,y_i)\}_{i=m+1}^n$.
\State $\hat{f}_{group} \leftarrow$ {\sc Train}$(\mathcal{D}_1, \mathcal{D}_2)$
\State $(\hat{z}_{m+1},\dots,\hat{z}_n) \leftarrow \textsc{Predict}(\hat{f}_{group},\mathcal{D}_2)$
\State $(\hat{z}_1,\dots,\hat{z}_m) \leftarrow (z_1,\dots,z_m)$
\State $\hat{f}_{robust} \leftarrow$  {\sc{Train\_Robust}}$\lt \{(x_i,y_i,\hat{z}_i)\}_{i=1}^n \rt$
\Ensure Final model  $\hat{f}_{robust}$.
\end{algorithmic}
  \end{algorithm}
    \vspace{-2em}
  \caption{General overview of \name, our algorithmic framework. \name{} can be instantiated with specific choices of methods to train the initial group classifier and the final robust model (see Algorithm \ref{alg:bbase} for example).}
  \vspace{-1em}
\end{figure}

\begin{figure}
\begin{algorithm}[H]
\caption{\name-Base}
\label{alg:bbase}
\begin{algorithmic}[1]
\Require Group-labeled data $\mathcal{D}_1 = \{(x_i,y_i,z_i)\}_{i=1}^m$, group-unlabeled data $\mathcal{D}_2 = \{(x_i,y_i)\}_{i=m+1}^n$.
\State $\hat{f}_{group} \leftarrow$ {\sc Train\_Supervised}$(\mathcal{D}_1)$
\State $(\hat{z}_{m+1},\dots,\hat{z}_n) \leftarrow \textsc{Predict}(\hat{f}_{group},\mathcal{D}_2)$
\State $(\hat{z}_1,\dots,\hat{z}_m) \leftarrow (z_1,\dots,z_m)$
\State $\hat{f}_{robust} \leftarrow$  {\sc{GDRO}}$\lt \{(x_i,y_i,\hat{z}_i)\}_{i=1}^n \rt$
\Ensure Final model  $\hat{f}_{robust}$.
\end{algorithmic}
  \end{algorithm}
    \vspace{-2em}
  \caption{The base version of \name, which we use in most experiments. A supervised classifier is trained on the group-labeled examples and generates pseudolabels for the remaining training data. These pseudolabels are used to train a more robust model using GDRO.
  The algorithm is described in more detail below, with pseudocode for subroutines in Appendix \ref{app:experiments}.
  \vspace{-1em}
  }
\end{figure}

To address the problem of improving group robustness when only some group labels are available, we propose \name, a two-stage framework which leverages the group-labeled examples to generate group ``pseudolabels'' for the remaining datapoints, and then uses these pseudolabels to train a robust model on the target task.
This two-stage approach is inspired by prior works such as \textsc{Jtt} and \textsc{George} \citep{liu2021just,sohoni2020no}. However, unlike these methods which assume all group labels are unknown, \name{} is capable of exploiting the additional information in the group labels that \emph{are} known for some datapoints. \name's overall workflow is illustrated in Algorithm \ref{alg:general}.

\paragraph{Stage 1: Predicting group labels (via a ``class-conditional'' classifier).} First, we train a model $f_{group} \in \mathcal{F}_{1}$ to \emph{predict} the group labels for the training and validation datapoints that do not have provided group labels. To do so, we train a supervised classifier on the training points with known group labels.\footnote{We train this group classifier with the GDRO objective, to encourage good accuracy at recognizing each group.} Despite the small number of these points, we show that this simple approach can perform surprisingly well with a key modification: we use the \emph{class} label (which is assumed known for all training datapoints) as an input to the group classifier, since the probabilities of each group can vary conditioned on the class.
Specifically, we compute the empirical probabilities of each group conditioned on the class, and compute corresponding logits. For each example that is fed into the group classifier, the logits for the appropriate class are summed with the output of the last layer of the network.
In the datasets we evaluate on, the groups are subsets of the classes, so this means that we effectively take the softmax over the logits output by the network over all groups \emph{belonging to the known class} to get the predicted per-group probabilities (assigning zero probability to all groups in different classes).\footnote{This can also be interpreted as a form of multi-task learning (MTL) with hard weight sharing, where each class corresponds to a task.}
In this way, the class information can help the model learn to distinguish the groups, offsetting the dearth of data. (In Section \ref{sec:no-class}, we evaluate the impact of this choice on the performance of the group classifier and the final model.)
We train $f_{group}$ with the GDRO loss to encourage good performance at predicting each group.

To select the best group classifier model over the course of training, we use a group-labeled subset of the validation set with the same size as the group-labeled training set (ensuring that the \emph{total} number of group labels required is small across both training and validation splits).\footnote{It is not fundamentally necessary for the two group-labeled subsets to be the same size; this is merely a simple heuristic to trade off the amount of data for training the group classifier, and the amount of data for model selection.}
This group classifier is then used to generate ``pseudo-group-labels'' for all training datapoints without a known group label. We term this approach (together with Stage 2) \name-Base. (Henceforth, where unspecified, \name{} refers to \name-Base.)

We find that this simple supervised approach works well and is relatively simple to analyze (and inexpensive to run), and therefore focus on it for the majority of this paper.
Nevertheless, a more complicated method could also be used in this stage: for instance, rather than this simple supervised approach, one could use semi-supervised learning to leverage the points without group labels and train an improved group classifier. We term this approach \name-SSL. As our preliminary experiments in Section \ref{sec:semi-supervised} and theoretical analysis in Section \ref{sec:analysis} show, the use of more sophisticated algorithms such as \name-SSL can indeed further improve performance (albeit at the cost of added complexity and runtime).

\paragraph{Stage 2: Training a robust model.} Intuitively, if the predicted group labels $\hat{z}_i$ from Stage 1 align closely enough with the (unobserved) true group labels $z_i$, then training a model to be robust with respect to the \emph{predicted} groups should also induce good robustness with respect to the \emph{true} groups. Following this logic, we train a model $f_{robust} \in \mathcal{F}_{2}$ on the original task using GDRO \citep{sagawa2019distributionally}, where the groups are defined by the predicted group labels from Stage 1 (except for the datapoints with ground-truth group labels provided, for which we use this ground-truth). The same small group-labeled validation subset as in Stage 1 is used for validation of this model.

\paragraph{Initialization.} For training both the group classifier and the robust model, we typically start from a pretrained model (e.g., ResNet-50 pretrained on ImageNet). We discuss this more in Section \ref{sec:pretraining}, wherein we evaluate the impact of using different pretrained models.

\section{Analysis}
\label{sec:analysis}
In this section, we analyze the theoretical worst-group performance of our approach. First, we state an upper bound on the worst-group loss of GDRO, and a \emph{lower} bound on that of ERM. Next, we prove a generalization bound on the worst-group loss of \name{} (Theorem \ref{thm:basic}), and discuss how this result relates to the GDRO and ERM bounds.

For notation, we use $\F_1, \F_2$ to denote the spaces of possible classifiers used in Stage 1 and Stage 2, respectively. We assume $\F_2$ is parameterized by $\theta \in \R^d$, so the \name{} model $\hat{f}_{robust}$ has parameters $\hat{\theta}_{robust}$. We define $\L_{robust}(\theta) := \max\limits_{g \in \mathcal{G}} \E_{(x,y)|z = g} [\ell(f(x; \theta), y)]$ (the worst-group loss), and $\L^*_{robust} := \min\limits_{\theta \in \R^d} \L_{robust}(\theta)$. 
Finally, let $q$ be the population proportion of the rarest group. In this section, we assume $\ell$ is either the squared loss between probabilities, or the truncated cross-entropy loss, so that $\ell$ is bounded and Lipschitz.

First, we show that if GDRO is trained on the dataset of the $m$ group-labeled points, the excess worst-group risk of the resulting model (compared to the worst-group-optimal model, i.e. the model in $\mathcal{F}_2$ with the lowest population worst-group performance on the task) scales as $\tilde{O}(1/\sqrt{qm})$.

\begin{restatable}{lemma}{gdro}
\label{lem:gdro}
Let $\hat{f}_{gdro,labeled} \in \mathcal{F}_2$ (with associated parameters $\hat{\theta}_{gdro,labeled}$) denote the GDRO classifier trained on $\mathcal{D}_1$ only. Then with high probability, $\L_{robust}(\hat{\theta}_{gdro,labeled}) \le \L^*_{robust} + \tilde{O}(\ff{1}{\sqrt{qm}})$.
\end{restatable}

By contrast, it is easy to show that ERM (even trained on the full dataset) can result in a worst-group risk multiple times higher than that of the optimal worst-group model:
\begin{restatable}{lemma}{erm}
\label{lem:erm}
Let $\hat{f}_{erm,full} \in \mathcal{F}_2$ (with associated parameters $\hat{\theta}_{erm,full}$) denote the ERM classifier trained on the full dataset. There exists a distribution $\mathcal{P}$ such that with high probability, $\L_{robust}(\hat{\theta}_{erm,full}) \ge \ff{\log(1/q)}{\log(|\Y|)} \L^*_{robust}- \tilde{O}(1/\sqrt{n})$.
\end{restatable}

We now seek to understand how \name{} generalizes. First, we relate the performance of the group classifier in the first stage to the excess worst-group risk of the end model.

\begin{restatable}{theorem}{basic}
\label{thm:basic}
Suppose that on \emph{each} group, the error rate of the group classifier from ``Stage 1" is $\le r$. Then with high probability, $\L_{robust}(\hat{\theta}_{robust}) \le \L^*_{robust} + \tilde{O} \lt \ff{r}{q} + \ff{1}{\sqrt{qn}} \rt$.
\end{restatable}

Theorem~\ref{thm:basic} says that the excess worst-group risk scales \emph{linearly} in the error rate of the group classifier, plus an additional $O(\ff{1}{\sqrt{qn}})$ term which is small if the \emph{total} number of datapoints is large. In particular, if we use standard learning-theoretic results to bound the error rate of the group classifier, we obtain the following Corollary~\ref{coro:base}.
For Corollary~\ref{coro:base} and the remainder of this section, for simplicity we shall additionally assume the group classification problem is realizable: specifically, we assume there exists $f^* \in \mathcal{F}_1$ such $\E_{(x,y,z) \sim \P}[\ell(f(x,y), z)] = 0$.\footnote{For standard losses such as (truncated) cross-entropy or squared loss, this implies $\argmax\{f(x,y)\} = z$ with probability 1.} We show how to relax this assumption in Theorem \ref{thm:bayes} (Appendix~\ref{app:theory}).

\begin{restatable}{corollary}{base}
\label{coro:base}
With high probability, for \name-\emph{Base} we have $\L_{robust}(\hat{\theta}_{robust}) \le \L^*_{robust} + \tilde{O} \lt \ff{1}{q\sqrt{m}} \rt$.
\end{restatable}

A strength of Theorem~\ref{thm:basic} and Corollary~\ref{coro:base} are that they do not require assumptions on the data distribution (unlike prior work such as \citep{sohoni2020no} which requires specific distributional assumptions to obtain generalization bounds).
However, the downside of Theorem~\ref{thm:basic} is that the bound is relatively weak unless the group classifier is known to perform near-optimally; indeed, Corollary \ref{coro:base} yields a slightly weaker asymptotic bound than GDRO on the labeled data alone. This result can be improved upon when further assumptions are made in order to guarantee a stronger bound on the group classification error. For instance, in \name-SSL, \emph{semi-supervised} learning is used to train the group classifier, leveraging the group-unlabeled points. With an appropriate choice of semi-supervised learning method (such as FixMatch \citep{sohn2020fixmatch}), if $m = \Omega(\sqrt{n})$, then under appropriate conditions, the worst-group generalization error rate bound of \name-SSL is $O(1/\sqrt{n})$. This result is based on the PAC-learning based results of \cite{balcan2009discriminative}, and requires the assumptions therein as well as realizability (the exact conditions are somewhat technical, and are deferred to the discussion in Appendix~\ref{app:theory}). This translates to an excess robust risk of $O(1/\sqrt{n})$ for the final model, as stated in Corollary~\ref{coro:ssl}.

\begin{restatable}{corollary}{ssl}
\label{coro:ssl}
Under appropriate conditions, for \name-\emph{SSL} we have $\L_{robust}(\hat{\theta}_{robust}) \le \L^*_{robust} + \tilde{O} \lt \ff{1}{q\sqrt{n}} + \ff{1}{qm} \rt$ with high probability.
\end{restatable}

Finally, another natural question is how the performance of \name-Base compares to that of ERM. We study this in Corollary~\ref{coro:strong}.

\begin{restatable}{corollary}{strong}
\label{coro:strong}
In addition to the assumptions of Theorem \ref{thm:basic}, suppose that
for all $f \in \mathcal{F}_2$, $\ell(f(x), y) \perp \argmax\{\hat{f}_{group}(x, y)\}\, \big|\, z$.
Let $\theta^*_{avg}$ be the minimizer of the population \emph{average} loss $\E_{(x,y)\sim P} [\ell(x, y; \theta)]$.
Then with high probability, for \name-\textrm{Base} we have $\L_{robust}(\hat{\theta}_{robust}) \le \min\lt \L^*_{robust}+\tilde{O}(\ff{1}{q\sqrt{m}}), \, \L_{robust}(\theta^*_{avg}) +\tilde{O}(\ff{1}{\sqrt{qn}})\rt$.
\end{restatable}

For comparison, the robust loss of the ERM model is upper bounded by $\L_{robust}(\theta^*_{avg}) +\tilde{O}(\ff{1}{\sqrt{qn}})$ with high probability.

 In words, Corollary \ref{coro:strong} says that if we can assume that the errors made by $f_{group}$ are ``random'' conditioned on the true group identity---i.e., they do not affect the distribution of the loss on the target task---then in addition to the bound of Theorem~\ref{thm:basic}, we can also guarantee that the worst-group loss is at least as good as that of the ERM model (plus small $\tilde{O}(1/\sqrt{qn})$ noise). Of course, this ``random error'' assumption is very strong; nevertheless, in Section \ref{sec:random-flip} we compare the performance of \name-Base and simulated ``group predictions'' with the same confusion matrix as those of \name-Base but random errors, and find that the \name{} predictions do not substantially degrade performance compared to these randomized predictions. Thus, we hypothesize that the errors made by \name-Base are ``sufficiently random'' to make the conditions of Corollary~\ref{coro:strong} hold approximately.

\begin{table*}[t]
\caption{Main results. Baselines that do not use group labels for training: ERM,  \textsc{EIIL}, \textsc{George}, \textsc{Jtt}. We compare to our method (\name), and to GDRO run on only the points with known group labels (Subset-GDRO), when there are 32 group-labeled examples provided from each group, in both the training and validation sets. Additional baselines, and results for different numbers of group-labeled points, are in Table \ref{tab:wg-results-app} (Appendix \ref{app:experiments}). We also compare to GDRO run on the full dataset (Full-GDRO), which can be roughly interpreted as an ``upper bound'' on the expected performance (since it requires all group labels to be known).
For U-MNIST, we use a LeNet-5; for all other tasks, we use a ResNet-50, starting from the pretrained PyTorch model checkpoint trained on the supervised ImageNet task.
}
\label{tab:wg-results}
\centering
{\small
{
\resizebox{\textwidth}{!}{\begin{tabular}{@{}lccccccccc@{}}
\toprule
Method              & \multicolumn{2}{c}{U-MNIST} & \multicolumn{2}{c}{Waterbirds} & \multicolumn{2}{c}{CelebA} & \multicolumn{2}{c}{U-CIFAR10} \\
Accuracy ($\%$)     & Worst-group       & Avg.    & Worst-group     & Avg.  & Worst-group & Avg.  & Worst-group           & Avg.         \\ \midrule
ERM              &  $93.4 \pm 0.5$  &  $99.2 \pm 0.0$  &  $60.6 \pm 3.3 $  &  $97.3 \pm 0.1$ &  $39.7 \pm 3.0$  &  $95.7 \pm 0.1$  &  $88.4 \pm 1.4$  &  $99.5 \pm 0.1$  \\
EIIL &  $\textbf{97.2} \pm 0.5$  &  $98.9 \pm 0.2$  &  $87.3 \pm 4.2$  &  $93.1 \pm 0.6$  &  $81.3 \pm 1.4$  &  $89.5 \pm 0.4$ &  $85.3 \pm 1.4$  &  $99.4 \pm 0.1$   \\
\textsc{George}  & $95.7 \pm 0.6$ &     $97.9 \pm 0.2$  &  $76.2 \pm 2.0$  &  $95.7 \pm 0.5$  &  $53.7 \pm 1.3$ &  $94.6 \pm 0.2$  &  ${93.4} \pm 5.8$  &  $98.9 \pm 0.3$  \\
\textsc{Jtt}     &  ${96.2} \pm 0.7$ &  $98.4 \pm 0.4$  &  $88.0 \pm 0.7$  & $91.7 \pm 0.8$  & $77.8 \pm 2.0$ & $87.2 \pm 1.2$  &  $89.0 \pm 4.7$  &  $94.6 \pm 1.3$  \\
\name-Base {\scriptsize(ours)}  &  ${96.9} \pm 0.9$ &  $99.1 \pm 0.3$   &  $\textbf{89.6} \pm 0.9$  &  $94.3 \pm 1.3$  & $\textbf{83.8} \pm 2.7$  &  $92.8 \pm 0.6$ &  $\textbf{94.5} \pm 1.1$ &  $98.9 \pm 0.3$ \\ 
\midrule
Full-GDRO             &  $98.6 \pm 0.2$  &  $99.1 \pm 0.1$ &  $90.9 \pm 0.2$ &  $92.8 \pm 0.2$  &  $89.3 \pm 0.9$  & $92.8 \pm 0.1$  &  $97.0 \pm 0.3$  &  $99.2 \pm 0.3$ \\
\end{tabular}}
}}
\end{table*}

\section{Experiments}
\label{sec:experiments}
We empirically validate that \name{} improves group robustness on four different image classification tasks. In Section \ref{sec:wg-results}, we study how the worst-group performance of \name{} scales with the number of group labels, and compare it with several baselines. We show that with as few as 1-33\% of points having group labels, \name{} attains better worst-group performance than the baselines that do not use this group information, and approaches the worst-group performance of GDRO on the full dataset as the number of group labels increases. We also show that \name{} always outperforms GDRO trained on only the subset of points with group labels. In Section \ref{sec:group-acc}, we confirm that the worst-group accuracy of the final model increases with the accuracy of stage 1 of \name. 
In Section \ref{sec:semi-supervised}, we show that using semi-supervised learning (SSL) for the first stage of \name{} can further improve worst-group accuracy.
In Section \ref{sec:ablations}, we present ablation experiments to better understand the effect of \name's design choices on the worst-group performance of the final model. Additional ablation experiments and experimental details are in Appendix \ref{app:experiments}.

In our experiments, we study how performance varies as we increase the number of group-labeled examples \emph{per group}. In other words, we pick a fixed budget of (training and validation) examples to label for each group.\footnote{We use the same group label budget for both the training and validation sets.} These examples are randomly sampled from the appropriate group in the original dataset. In reality, if we only have access to a balanced group-labeled dataset (e.g., one with the same group proportions as the population), we can convert it to a balanced one by simply subsampling the group-labeled points by group to get a (smaller) balanced dataset; thus, our performance can be interpreted as a rough lower bound on the performance without such subsampling. (In fact, we found that this subsampling does not meaningfully degrade the final performance, compared to using more group-labeled points for the larger groups.)

\subsection{Datasets}
\label{sec:datasets}
We evaluate on four image classification tasks.
\newline
\textbf{U-MNIST.} U-MNIST \citep{sohoni2020no} is a modified version of MNIST \citep{lecun2010mnist}, where the task is to classify digits as `$<5$' or `$\ge 5$', the groups are the individual digits, and only 5\% of images in the `8' group are retained from the individual dataset. This rarity makes `8' images more difficult to classify.
\newline
\textbf{Waterbirds.} Waterbirds \citep{sagawa2019distributionally}, is a popular robustness benchmark that consists of images from `landbird' and `waterbird' species on either land or water backgrounds. The task is to classify images as `landbird' vs. `waterbird', and the groups are defined by background. 95\% of landbirds are on land backgrounds and similarly for waterbirds; this spurious correlation makes landbirds on water and waterbirds on land harder to classify.
\newline
\textbf{CelebA.} CelebA \citep{liu2015deep} is a popular face classification dataset often used to evaluate robustness. The task is to classify faces as `blond' or `not blond', and the groups are defined by gender. Only 6\% of blond examples are male, leading to poor performance on this group.
\newline
\textbf{U-CIFAR10.} We introduce U-CIFAR10 as a modification of the CIFAR-10 dataset \citep{krizhevsky2009learning}, where the task is to classify the image as `animal' or `vehicle', the groups are the 10 original CIFAR-10 classes, and we undersample the `airplane' class to 5\%. Though similar to U-MNIST, this task is much more challenging.

\subsection{Results: Worst-Group Performance}
\label{sec:wg-results}

\begin{figure*}
    \centering
    \includegraphics[width=0.4\textwidth]{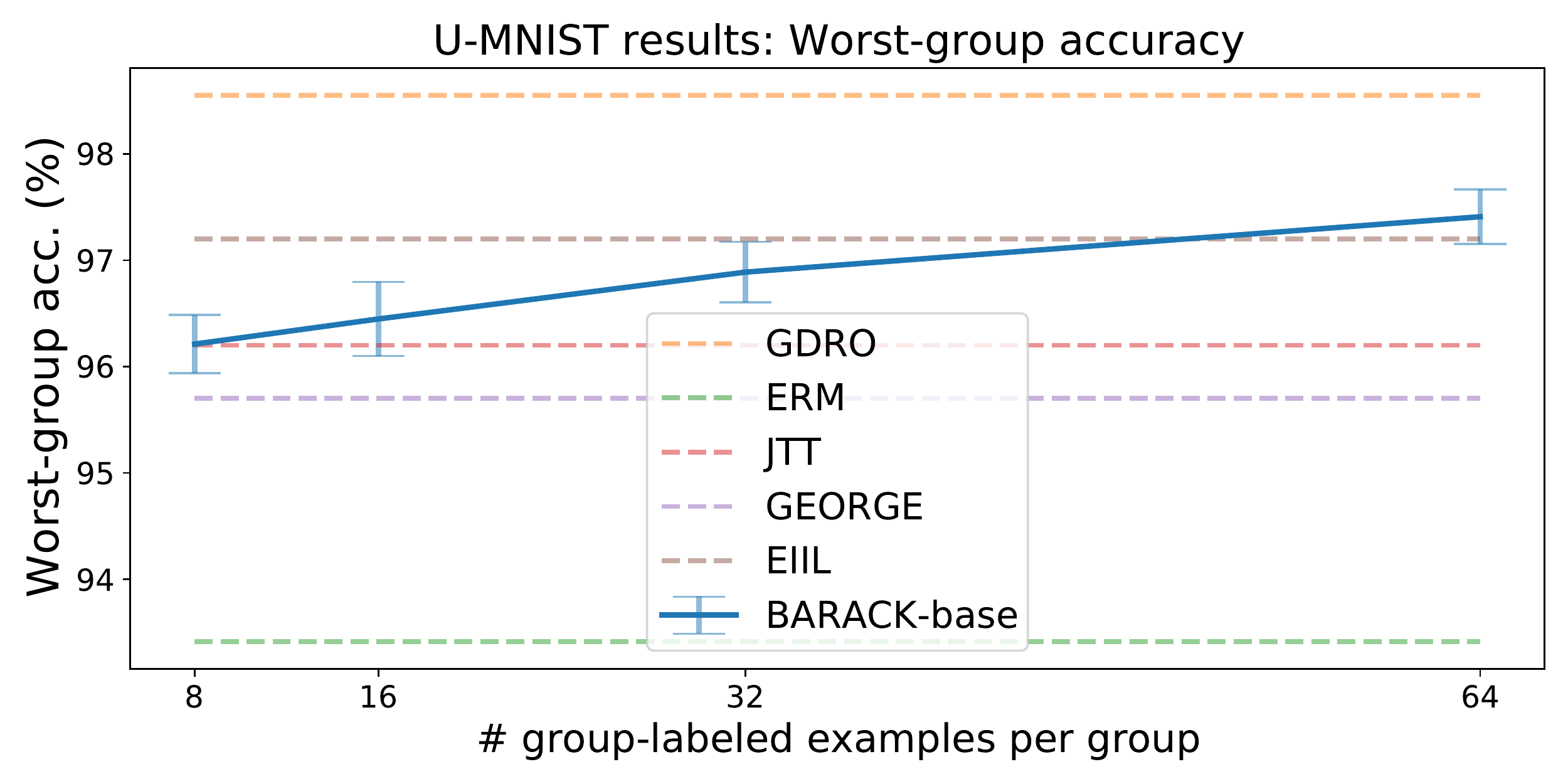}\quad\includegraphics[width=0.4\textwidth]{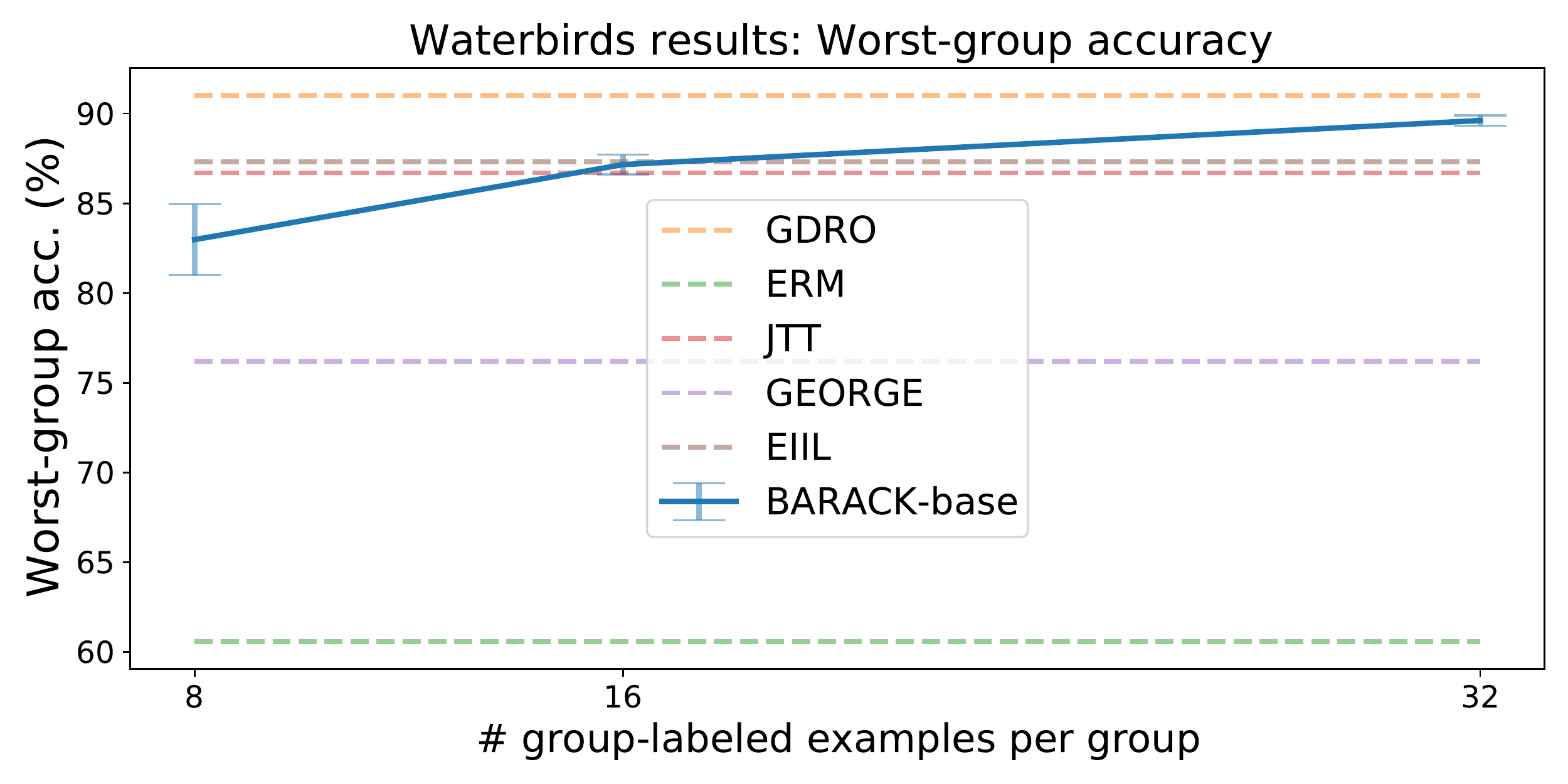}
    \\
    \includegraphics[width=0.4\textwidth]{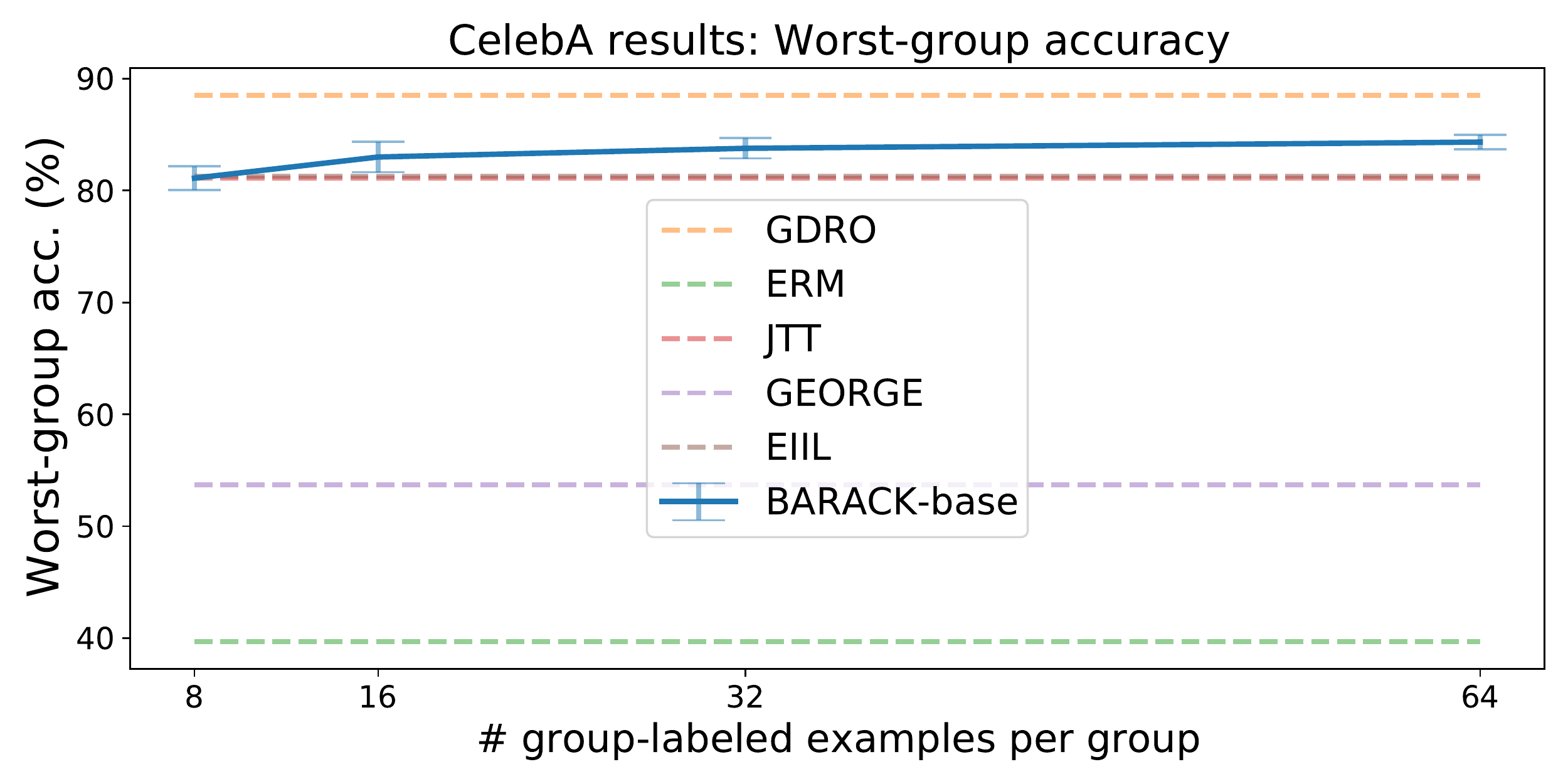}\quad\includegraphics[width=0.4\textwidth]{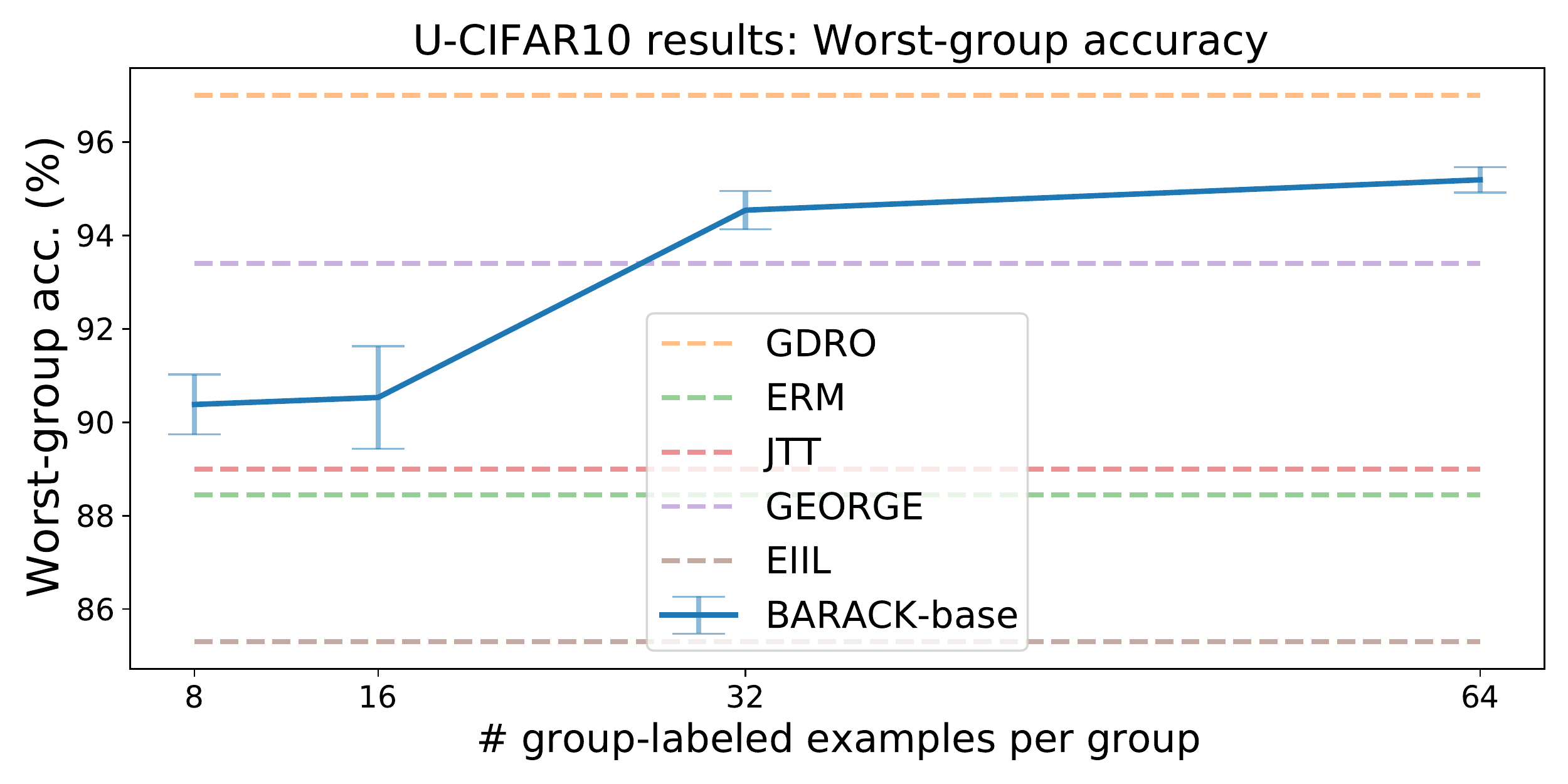}
    \vspace{-1em}
    \caption{Worst-group accuracy as a function of number of group-labeled examples. x-axis denotes the number of training examples with group labels, per group (which equals the number of validation group-labeled examples per group). (Subset-GDRO plots in Appendix \ref{app:experiments}.)}
    \label{fig:wg-results}
\end{figure*}

Across the four datasets in Section \ref{sec:datasets}, \name{} matches or improves worst-group accuracy compared to baselines that do not use group information; results are presented in Table~\ref{tab:wg-results} and Figure \ref{fig:wg-results}. The baselines that we study are ERM, \textsc{George} \citep{sohoni2020no} (which do not use group information), and \textsc{Eiil} \citep{creager2021environment} and \textsc{Jtt} \citep{liu2021just} (which both use group information on the validation set only). We provide additional baseline results in Appendix \ref{app:experiments}; for instance, we also compare to the baseline of using GDRO on only the points with group labels (Subset-GDRO).

With as few as 1\%-33\% of training datapoints having group labels,
\name{} improves over all these baselines (Figure \ref{fig:wg-results}). For instance, CelebA has 1387 training points in the smallest group, and \name{} outperforms the baselines on CelebA with 16 group-labeled training examples per group. As the number of group-labeled points increases, the worst-group performance of \name{} gets closer to that of GDRO trained with all the group labels.

In terms of \emph{average} accuracy, \name{} and full-dataset GDRO are typically similar, while ERM is usually somewhat higher. This is to be expected, since on these tasks there is a tradeoff between optimizing for average-case and worst-case performance, as previously observed in the literature \citep{sagawa2019distributionally}. \name{} also substantially outperforms Subset-GDRO (GDRO trained only on the subset of group-labeled points), in terms of both worst-group and average accuracy. Subset-GDRO fails to generalize well due to the limited amount of training data it uses.

\begin{table*}[t]
\caption{\emph{Group prediction} accuracies for stage 1 of \name, for the same settings as in Table \ref{tab:wg-results}. Additional results \& plots are in Appendix \ref{app:experiments}.}
\label{tab:group-pred}
\centering
{\small
\resizebox{\textwidth}{!}{\begin{tabular}{@{}lccccccccc@{}}
\toprule
Method    & \multicolumn{2}{c}{U-MNIST}          & \multicolumn{2}{c}{Waterbirds} & \multicolumn{2}{c}{CelebA} & \multicolumn{2}{c}{U-CIFAR10} \\
    & Worst-group       & Avg.  & Worst-group       & Avg.    & Worst-group     & Avg.  & 
Worst-group & Avg.         \\ \midrule
Accuracy ($\%$)     & $87.7 \pm5.4$ & $94.6 \pm 1.2$ &  $91.5\pm 1.1$  &  $93.1 \pm 0.8$  & $85.9 \pm 0.3$ &  $89.4 \pm 0.6$  &  $67.9 \pm 2.6$ &  $83.5 \pm 0.6$  \\ 
\midrule
\end{tabular}}
}
\end{table*}

\subsection{Results: Group Prediction Accuracy}
\label{sec:group-acc}
In this section, we study the performance of Stage 1 of \name{} (accuracy at predicting the group labels), in order to better understand the performance of Stage 2 (worst-group accuracy on the target classification task).
In Table \ref{tab:group-pred}, we report the group prediction accuracies on each dataset, i.e., the accuracies of the predicted group labels from Stage 1 of \name.

Table \ref{tab:group-pred} shows that the group prediction models are far from perfect. Indeed, on U-CIFAR10 the accuracy at predicting the `airplane' group is below 70\%, and the average accuracy over all groups is only $\approx84\%$. Surprisingly, this only causes a modest drop in performance for the final robust model, as seen in Table \ref{tab:wg-results}. This can be explained with the help of the intuition from Corollary \ref{coro:strong}: if the errors made by the group prediction model are ``sufficiently random'' (i.e., not adversarial), we should expect the worst-group performance of \name{} to improve upon that of ERM, and approach that of GDRO as the number of group-labeled points increases. We explore this further in Section \ref{sec:random-flip}.

\subsection{\name-\textrm{SSL}: Semi-Supervised Learning for Group Prediction}
\label{sec:semi-supervised}
As an extension to demonstrate the flexibility of the \name~framework, we investigate the use of semi-supervised learning (SSL) using FixMatch \citep{sohn2020fixmatch}, for Stage 1 of \name. (We still use GDRO for Stage 2.) We refer to this procedure as \name-SSL. On the U-CIFAR10 task, \name-SSL can improve group prediction accuracy and, correspondingly, final robust performance. For example, with only 8 group-labeled examples per group, the worst-group accuracy of the final \name-SSL model is \textbf{94.0}\%, compared to 90.4\% for \name-Base. Correspondingly, the worst-group prediction accuracy of the group classifier is 83.7\% when trained using FixMatch, while it is substantially lower at 42.9\% when trained using simple supervised learning as in \name-Base, which helps explain these results. Thus, while \name-Base is simple and attains good worst-group accuracy, these results highlight the exciting potential of using more advanced SSL techniques to further boost worst-group performance of \name~{at the cost of more complexity}.
(Note: When training the group classifier using FixMatch, we use the \emph{class} label as an input to the prediction head, just as in \name-Base.)

\subsection{Ablation Experiments}
\label{sec:ablations}
In this section, we present ablation experiments to study the reasons behind the worst-group accuracy gains offered by \name. First, in Section~\ref{sec:random-flip} we run a synthetic experiment in which we run GDRO with randomly generated noisy group labels at different noise levels, to better understand how group prediction errors affect the final robust performance.
Next, in Section~\ref{sec:no-class} we ablate the importance of using the class label as an input to the group prediction model (as described in Section~\ref{sec:method}.)
Finally, in Section~\ref{sec:pretraining} we explore using models that are pretrained on ImageNet in a \emph{self-supervised} manner (instead of supervised) as the starting model for \name~(and all the baselines), to understand how the worst-group accuracy trends from Section~\ref{sec:wg-results} generalize when different pretrained models are used. 

\begin{figure*}
    \centering
    \includegraphics[width=0.4\textwidth]{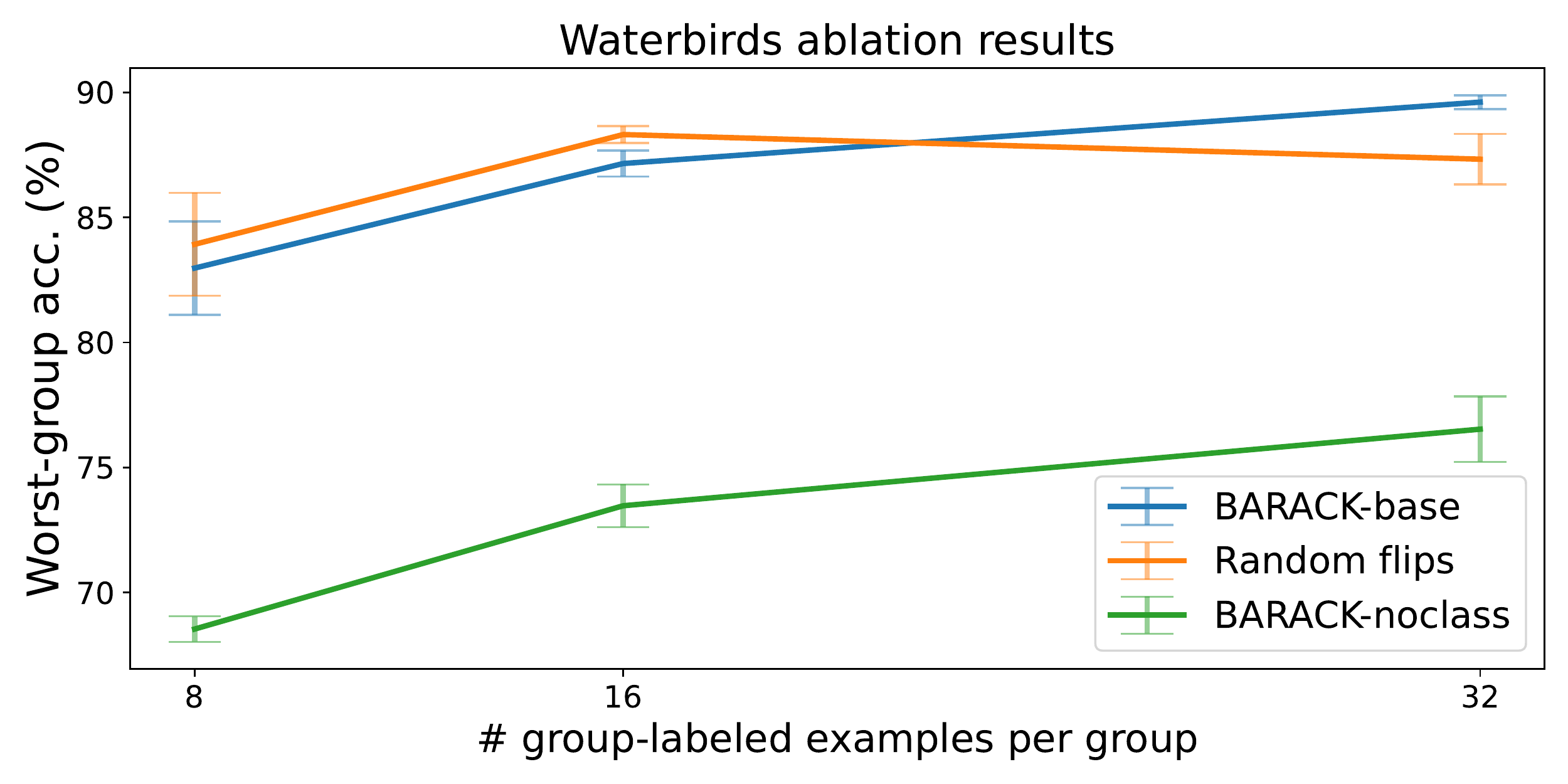}\quad\includegraphics[width=0.4\textwidth]{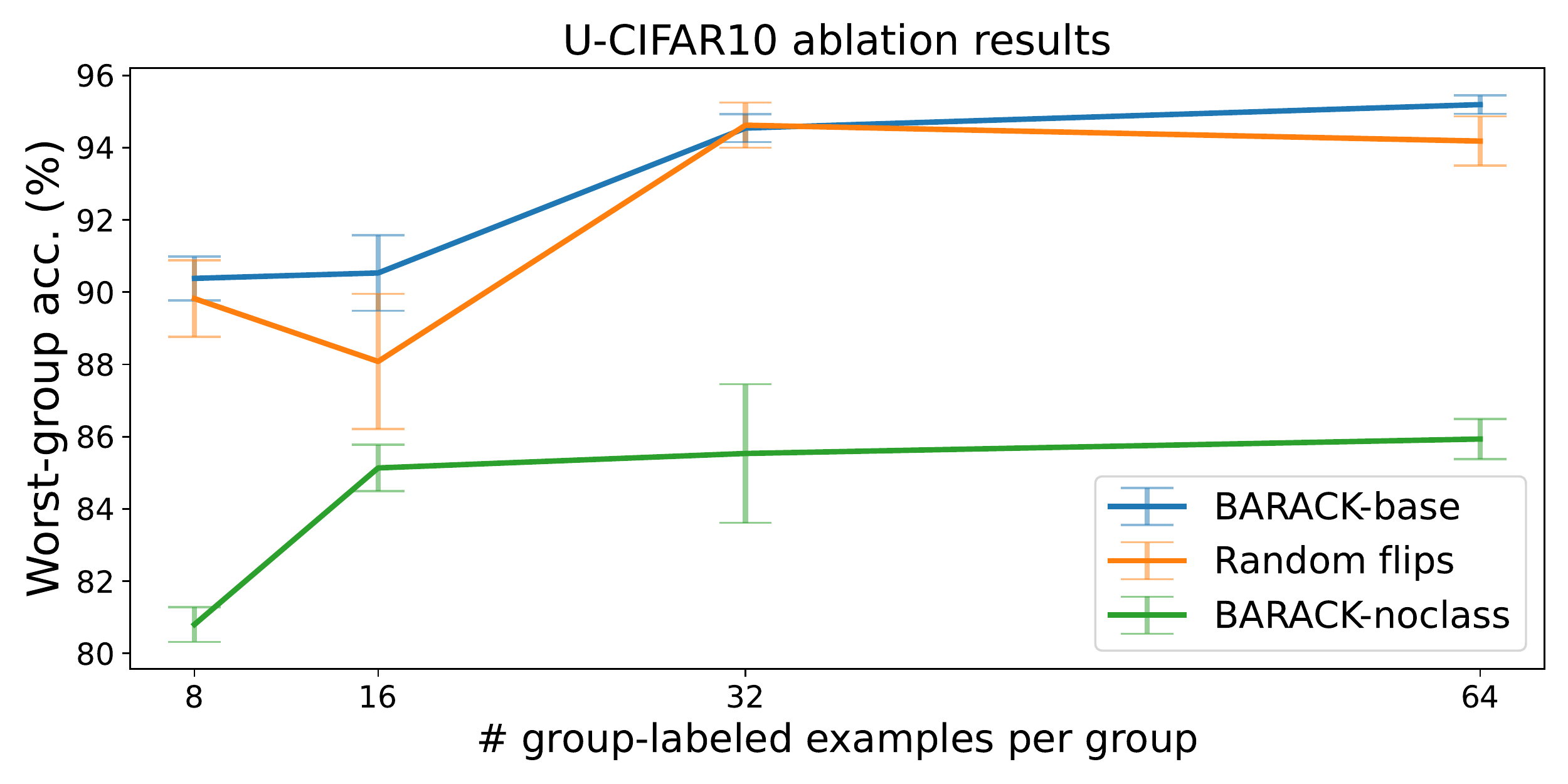}
\\
\includegraphics[width=0.42\textwidth]{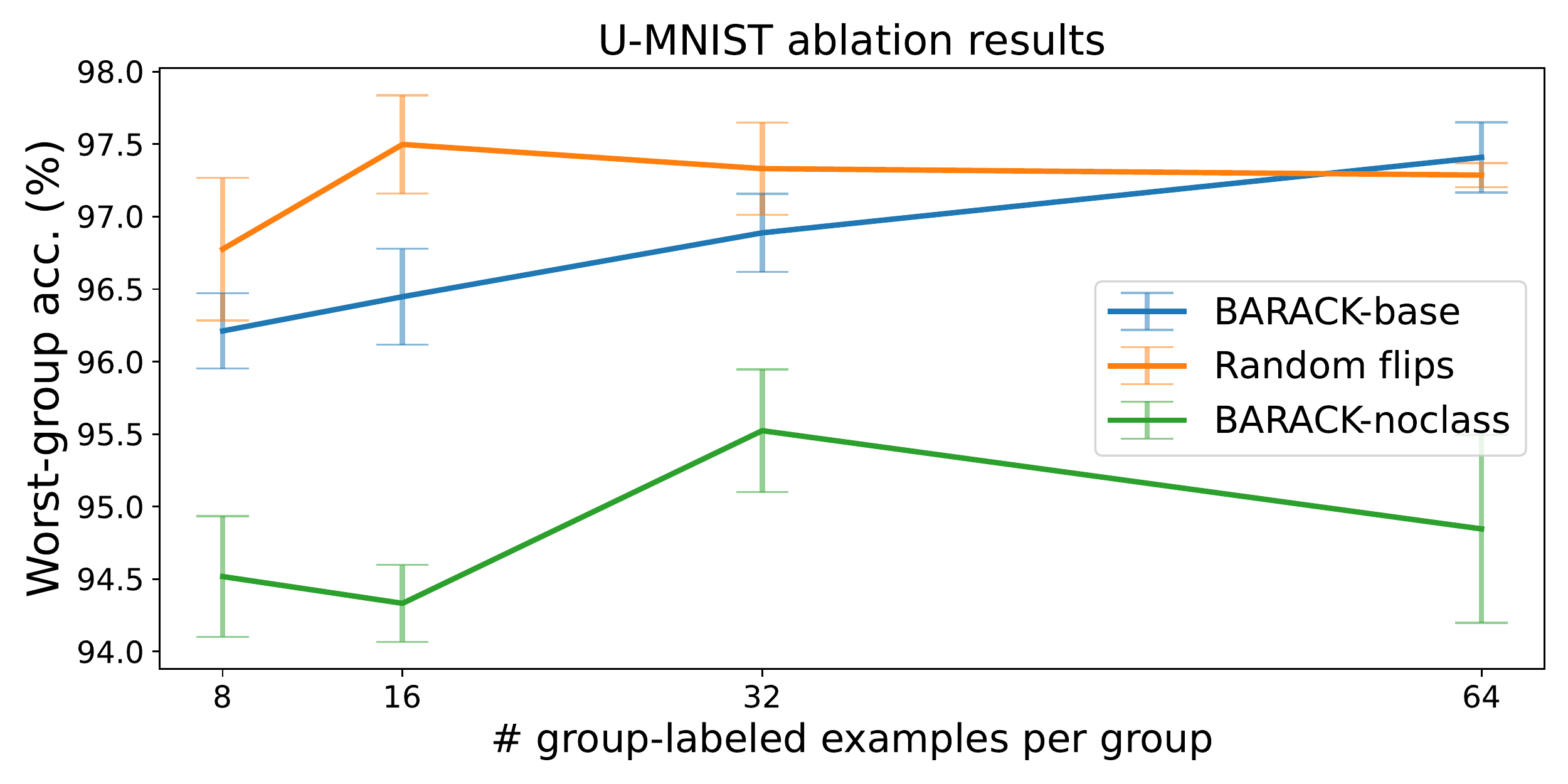}\quad\includegraphics[width=0.42\textwidth]{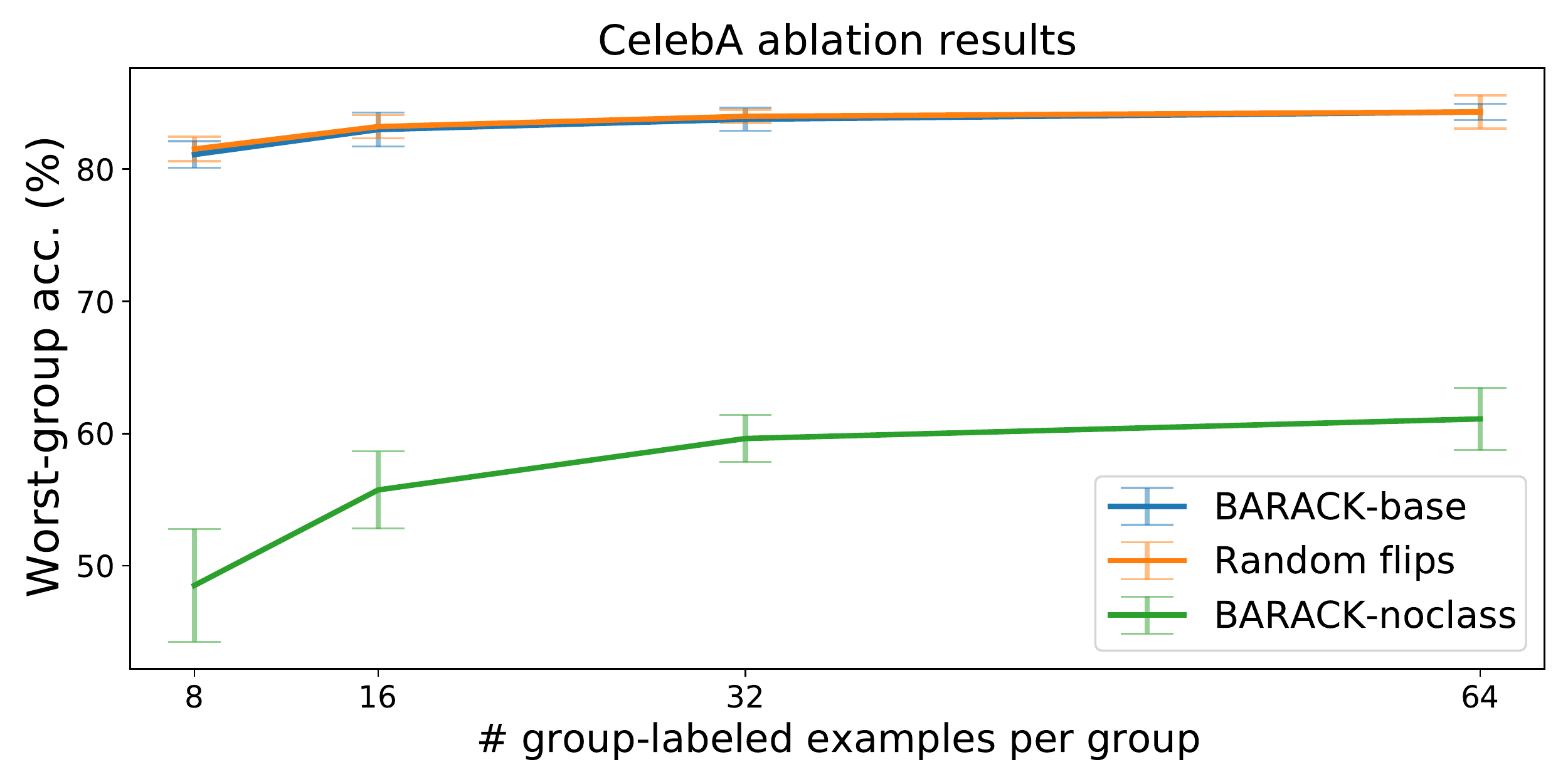}
    \caption{Worst-group accuracy for random flipping experiment (Section~\ref{sec:random-flip}, orange line) and no-class experiment (Section~\ref{sec:no-class}, green line).}
    \label{fig:ablations}
\end{figure*}

\subsubsection{Random Flipping}
\label{sec:random-flip}
To better understand why \name{} can achieve worst-group accuracy close to that of GDRO even with fairly inaccurate predicted group labels, we run a synthetic experiment: we run GDRO with a varying fraction of the group labels randomly flipped, and compare the performance of this to that of \name{} at an equivalent group prediction error rate. Results are in Figure~\ref{fig:ablations}; GDRO with the randomly perturbed group labels performs similarly to \name.

Specifically, we take the ground-truth group labels, randomly flip them to get the same error rate and confusion matrix as the predicted group labels from \name's group prediction model, and then use these ``noisy group labels'' as the groups for GDRO. For most settings, the final worst-group error is quite similar to that of \name, which suggests that the errors made by our group prediction model are indeed ``sufficiently random'' to not adversely affect the downstream worst-group accuracy too much. However, on U-MNIST, \name{} does underperform the random flipping version with a small number of group-labeled points, although this difference decreases as the number of group-labeled points increases. Further details on the random flipping experiments are in Appendix \ref{app:experiments}.

\subsubsection{Class Label Input}
\label{sec:no-class}
To evaluate the importance of using the \emph{class} label as input in the first stage of \name{}, we run the standard \name{} procedure except without using the class label as input. This substantially decreases the final worst-group accuracy on all datasets  (often by 10\% or more); see Figure~\ref{fig:ablations}. On most datasets, this drop can largely be explained by the reduced group prediction accuracy when the class feature is not used (the worst-group accuracy of the group classifier is up to 20 points worse when the class is not used). Interestingly, on U-CIFAR10, the group classifier's worst-group accuracy drops only 1-5\% when it does not use the class label, but this still results in a drop of 8-11\% in the worst-group accuracy of the final \name{} model.
(Additional results and plots can be found in Appendix \ref{app:experiments}.)

Note that for all datasets considered, the classes are disjoint unions of the groups, meaning that knowing the class narrows down the possibilities for the group label. Thus, it is unsurprising that using the class label in the group prediction model significantly improves the worst-group accuracy of both the group prediction model and the final model. We hypothesize that the class label is more essential for group prediction in the spurious correlation setting (as suggested by the results in Figure~\ref{fig:ablations}), because it enables the group prediction model to disambiguate between examples with the same spurious attribute but different classes, allowing it to focus on identifying the spurious attribute itself. By contrast, there is no clear ``spurious attribute'' on U-MNIST and U-CIFAR10; knowing the class label only reduces the number of candidate group labels for each example from 10 to 5.

\subsubsection{Pretrained Model Choice}
\begin{table*}[t]
\caption{Results with self-supervised pretrained model (RotNet), for 32 group-labeled examples per group as in Table \ref{tab:wg-results}. Worst-group accuracies are shown with and without the use of data augmentation. Additional results on pretraining are in Appendix \ref{app:experiments}. }
\label{tab:rotnet}
\centering
{\small
\begin{tabular}{@{}lccccccccc@{}}
\toprule
Method              & \multicolumn{2}{c}{Waterbirds} & \multicolumn{2}{c}{CelebA} & \multicolumn{2}{c}{U-CIFAR10} \\
Worst-group Acc. ($\%$)     & No aug.       & Aug.    & No aug.     & Aug.  &  No aug.     & Aug.
\\ \midrule
ERM          &  $38.4 \pm 2.2$  &  $56.0 \pm 3.8$ &  $39.3 \pm 1.0$  &  $39.3 \pm 2.1$  &  $79.3 \pm 1.2$ &  $82.3 \pm 2.3$  \\
\name{}-Base    &  $73.3 \pm 0.6$  &  $83.0 \pm 1.1$  & $84.4 \pm 1.5$ & $84.6 \pm 1.3$ &  $88.9 \pm 1.5$ &  $89.7 \pm 0.6$ \\ 
GDRO (full dataset)       &  $79.6 \pm 2.5$  &  $85.1 \pm 0.8$  &  $82.6 \pm 1.5$ &  $89.8 \pm 0.5$ &  $93.6 \pm 0.9$ &  $96.4 \pm 0.4$  \\
\midrule
\end{tabular}
}
\end{table*}

\label{sec:pretraining}
As described in Section \ref{sec:method}, on all tasks except U-MNIST, we start from a pretrained model trained on the supervised ImageNet task for all methods (and for both the group classifier and final model in \name{}). This is the standard approach when training on Waterbirds and CelebA, as in previous works \citep{sagawa2019distributionally,sohoni2020no,levy2020large,liu2021just}.
In this subsection, we investigate how the choice of pretrained model affects performance (Table \ref{tab:rotnet}). One motivation for this experiment is the potential for overlap or ``leakage" between ImageNet and other image classification datasets \citep{kolesnikov2019bit}. 
Moreover, we seek to confirm that the observations in previous sections regarding the performance of \name, compared to the baselines, are robust to the choice of pretrained model initialization. Thus, in this section we use pretrained models that were trained only with self-supervision (i.e., not using any labels). To be specific, in this section we use RotNet \citep{gidaris2018unsupervised} from the VISSL library \citep{goyal2021vissl} as opposed to the ResNet-50 pretrained on supervised ImageNet (from PyTorch) used in previous sections and previous works. Results are reported in Table~\ref{tab:rotnet}.

Overall, we observe that using the RotNet model achieves somewhat worse accuracies (both worst-group and average) for \emph{all} methods, although  this gap can be reduced or eliminated by using data augmentation.
This aligns with the findings of the original RotNet paper \citep{gidaris2018unsupervised} that RotNet has somewhat worse transfer performance compared to supervised pretrained models.
Despite the gap, the key takeaway is that \name{} still outperforms the baselines and remains competitive with full-dataset GDRO with the same initialization. 

\section{Conclusion}
We present \name, a two-stage approach to improve group robustness in the setting when only a small number of group labels are known. We empirically validate that \name{} outperforms methods that do not use training group labels, even with just a small number of group-labeled examples. We theoretically provide generalization bounds on the worst-group performance of \name{}. Our results indicate that even a small number of group labels can be helpful for substantially improving worst-group performance.

\newpage
\bibliographystyle{abbrvnat}
\bibliography{barack.bib}

\begin{thebibliography}{51}
\providecommand{\natexlab}[1]{#1}
\providecommand{\url}[1]{\texttt{#1}}
\expandafter\ifx\csname urlstyle\endcsname\relax
  \providecommand{\doi}[1]{doi: #1}\else
  \providecommand{\doi}{doi: \begingroup \urlstyle{rm}\Url}\fi

\bibitem[Agarwal et~al.(2018)Agarwal, Beygelzimer, Dud{\'\i}k, Langford, and
  Wallach]{agarwal2018reductions}
A.~Agarwal, A.~Beygelzimer, M.~Dud{\'\i}k, J.~Langford, and H.~Wallach.
\newblock A reductions approach to fair classification.
\newblock In \emph{International Conference on Machine Learning (ICML)}, pages
  60--69. PMLR, 2018.

\bibitem[Balashankar et~al.(2019)Balashankar, Lees, Welty, and
  Subramanian]{balashankar2019what}
A.~Balashankar, A.~Lees, C.~Welty, and L.~Subramanian.
\newblock What is fair? {Exploring} {Pareto}-efficiency for fairness
  constrained classifiers.
\newblock \emph{arXiv preprint arXiv:1910.14120}, 2019.

\bibitem[Balcan and Blum(2009)]{balcan2009discriminative}
M.-F. Balcan and A.~Blum.
\newblock A discriminative model for semi-supervised learning.
\newblock \emph{Journal of the ACM (JACM)}, 2009.

\bibitem[Ben-Tal et~al.(2013)Ben-Tal, Den~Hertog, De~Waegenaere, Melenberg, and
  Rennen]{ben2013robust}
A.~Ben-Tal, D.~Den~Hertog, A.~De~Waegenaere, B.~Melenberg, and G.~Rennen.
\newblock Robust solutions of optimization problems affected by uncertain
  probabilities.
\newblock \emph{Management Science}, 59\penalty0 (2):\penalty0 341--357, 2013.

\bibitem[Bhatia et~al.(2015)Bhatia, Jain, and Kar]{bhatia2015robust}
K.~Bhatia, P.~Jain, and P.~Kar.
\newblock Robust regression via hard thresholding.
\newblock \emph{arXiv preprint arXiv:1506.02428}, 2015.

\bibitem[Buolamwini and Gebru(2018)]{buolamwini2018gender}
J.~Buolamwini and T.~Gebru.
\newblock Gender shades: Intersectional accuracy disparities in commercial
  gender classification.
\newblock In S.~A. Friedler and C.~Wilson, editors, \emph{Conference on
  Fairness, Accountability and Transparency, FAT 2018, 23-24 February 2018, New
  York, NY, USA}, volume~81 of \emph{Proceedings of Machine Learning Research},
  pages 77--91. PMLR, 2018.
\newblock URL \url{http://proceedings.mlr.press/v81/buolamwini18a.html}.

\bibitem[Caron et~al.(2020)Caron, Misra, Mairal, Goyal, Bojanowski, and
  Joulin]{swav}
M.~Caron, I.~Misra, J.~Mairal, P.~Goyal, P.~Bojanowski, and A.~Joulin.
\newblock Unsupervised learning of visual features by contrasting cluster
  assignments.
\newblock In \emph{Advances in Neural Information Processing Systems
  (NeurIPS)}, 2020.

\bibitem[Chen et~al.(2020)Chen, Kornblith, Norouzi, and Hinton]{simclr}
T.~Chen, S.~Kornblith, M.~Norouzi, and G.~Hinton.
\newblock A simple framework for contrastive learning of visual
  representations.
\newblock In \emph{International Conference on Machine Learning (ICML)}, 2020.

\bibitem[Creager et~al.(2021)Creager, Jacobsen, and
  Zemel]{creager2021environment}
E.~Creager, J.-H. Jacobsen, and R.~Zemel.
\newblock Environment inference for invariant learning.
\newblock In \emph{International Conference on Machine Learning (ICML)}, 2021.

\bibitem[Duchi et~al.(2019)Duchi, Hashimoto, and
  Namkoong]{duchi2019distributionally}
J.~C. Duchi, T.~Hashimoto, and H.~Namkoong.
\newblock Distributionally robust losses against mixture covariate shifts.
\newblock \emph{Operations Research}, 2019.

\bibitem[Gebru et~al.(2017)Gebru, Krause, Deng, and Fei-Fei]{gebru2017scalable}
T.~Gebru, J.~Krause, J.~Deng, and L.~Fei-Fei.
\newblock Scalable annotation of fine-grained categories without experts.
\newblock In \emph{Conference on Human Factors in Computing Systems}, 2017.

\bibitem[Gidaris et~al.(2018)Gidaris, Singh, and
  Komodakis]{gidaris2018unsupervised}
S.~Gidaris, P.~Singh, and N.~Komodakis.
\newblock Unsupervised representation learning by predicting image rotations.
\newblock In \emph{International Conference on Learning Representations
  (ICLR)}, 2018.

\bibitem[Goel et~al.(2020)Goel, Gu, Li, and R{\'e}]{goel2020model}
K.~Goel, A.~Gu, Y.~Li, and C.~R{\'e}.
\newblock Model patching: Closing the subgroup performance gap with data
  augmentation.
\newblock In \emph{International Conference on Learning Representations}, 2020.

\bibitem[Goyal et~al.(2021)Goyal, Duval, Reizenstein, Leavitt, Xu, Lefaudeux,
  Singh, Reis, Caron, Bojanowski, Joulin, and Misra]{goyal2021vissl}
P.~Goyal, Q.~Duval, J.~Reizenstein, M.~Leavitt, M.~Xu, B.~Lefaudeux, M.~Singh,
  V.~Reis, M.~Caron, P.~Bojanowski, A.~Joulin, and I.~Misra.
\newblock Vissl.
\newblock \url{https://github.com/facebookresearch/vissl}, 2021.

\bibitem[Hardt et~al.(2016)Hardt, Price, and Srebro]{hardt2016equality}
M.~Hardt, E.~Price, and N.~Srebro.
\newblock Equality of opportunity in supervised learning.
\newblock \emph{Advances in Neural Information Processing Systems (NeurIPS)},
  29:\penalty0 3315--3323, 2016.

\bibitem[He et~al.(2016)He, Zhang, Ren, and Sun]{resnet}
K.~He, X.~Zhang, S.~Ren, and J.~Sun.
\newblock Deep residual learning for image recognition.
\newblock In \emph{Conference on Computer Vision and Pattern Recognition
  (CVPR)}, 2016.

\bibitem[Huber(1992)]{huber1992robust}
P.~J. Huber.
\newblock Robust estimation of a location parameter.
\newblock In \emph{Breakthroughs in Statistics}, pages 492--518. Springer,
  1992.

\bibitem[Kingma and Ba(2015)]{adam}
D.~P. Kingma and J.~Ba.
\newblock Adam: A method for stochastic optimization.
\newblock In \emph{International Conference on Learning Representations}, 2015.

\bibitem[Kolesnikov et~al.(2020)Kolesnikov, Beyer, Zhai, Puigcerver, Yung,
  Gelly, and Houlsby]{kolesnikov2019bit}
A.~Kolesnikov, L.~Beyer, X.~Zhai, J.~Puigcerver, J.~Yung, S.~Gelly, and
  N.~Houlsby.
\newblock Big transfer ({BiT}): General visual representation learning.
\newblock In \emph{European Conference on Computer Vision (ECCV)}, 2020.

\bibitem[Krizhevsky(2009)]{krizhevsky2009learning}
A.~Krizhevsky.
\newblock Learning multiple layers of features from tiny images.
\newblock 2009.
\newblock URL \url{http://www.cs.toronto.edu/~kriz/cifar.html}.

\bibitem[LeCun et~al.(1998)LeCun, Bottou, Bengio, and Haffner]{lecun1998}
Y.~LeCun, L.~Bottou, Y.~Bengio, and P.~Haffner.
\newblock Gradient-based learning applied to document recognition.
\newblock 1998.

\bibitem[LeCun et~al.(2010)LeCun, Cortes, and Burges]{lecun2010mnist}
Y.~LeCun, C.~Cortes, and C.~Burges.
\newblock {MNIST} handwritten digit database.
\newblock \emph{ATT Labs [Online]. http://yann.lecun.com/exdb/mnist}, 2010.

\bibitem[Lee et~al.(2013)]{lee2013pseudo}
D.-H. Lee et~al.
\newblock Pseudo-label: The simple and efficient semi-supervised learning
  method for deep neural networks.
\newblock In \emph{ICML Workshop on challenges in representation learning},
  page 896, 2013.

\bibitem[Levy et~al.(2020)Levy, Carmon, Duchi, and Sidford]{levy2020large}
D.~Levy, Y.~Carmon, J.~C. Duchi, and A.~Sidford.
\newblock Large-scale methods for distributionally robust optimization.
\newblock In \emph{Advances in Neural Information Processing Systems
  (NeurIPS)}, volume~33, pages 8847--8860, 2020.

\bibitem[Li et~al.(2019)Li, Sanjabi, Beirami, and Smith]{li2019fair}
T.~Li, M.~Sanjabi, A.~Beirami, and V.~Smith.
\newblock Fair resource allocation in federated learning.
\newblock \emph{arXiv preprint arXiv:1905.10497}, 2019.

\bibitem[Li et~al.(2021)Li, Beirami, Sanjabi, and Smith]{li2020tilted}
T.~Li, A.~Beirami, M.~Sanjabi, and V.~Smith.
\newblock Tilted empirical risk minimization.
\newblock In \emph{International Conference on Learning Representations
  (ICLR)}, 2021.

\bibitem[Liang and Ma(2019)]{229t}
P.~Liang and T.~Ma.
\newblock {CS 229T} course notes.
\newblock 2019.
\newblock URL \url{http://web.stanford.edu/class/cs229t/}.

\bibitem[Liu et~al.(2021)Liu, Haghgoo, Chen, Raghunathan, Koh, Sagawa, Liang,
  and Finn]{liu2021just}
E.~Liu, B.~Haghgoo, A.~S. Chen, A.~Raghunathan, P.~W. Koh, S.~Sagawa, P.~Liang,
  and C.~Finn.
\newblock Just train twice: Improving group robustness without training group
  information.
\newblock \emph{International Conference on Machine Learning}, 2021.

\bibitem[Liu et~al.(2015{\natexlab{a}})Liu, Luo, Wang, and Tang]{liu2015deep}
Z.~Liu, P.~Luo, X.~Wang, and X.~Tang.
\newblock Deep learning face attributes in the wild.
\newblock In \emph{Proceedings of the IEEE International Conference on Computer
  Vision (ICCV)}, pages 3730--3738, 2015{\natexlab{a}}.

\bibitem[Liu et~al.(2015{\natexlab{b}})Liu, Luo, Wang, and
  Tang]{liu2015faceattributes}
Z.~Liu, P.~Luo, X.~Wang, and X.~Tang.
\newblock Deep learning face attributes in the wild.
\newblock In \emph{Proceedings of International Conference on Computer Vision
  (ICCV)}, December 2015{\natexlab{b}}.

\bibitem[Lokhande et~al.(2022)Lokhande, Sohn, Yoon, Udell, Lee, and
  Pfister]{lokhande2022towards}
V.~S. Lokhande, K.~Sohn, J.~Yoon, M.~Udell, C.-Y. Lee, and T.~Pfister.
\newblock Towards group robustness in the presence of partial group labels.
\newblock \emph{arXiv preprint arXiv:2201.03668}, 2022.

\bibitem[Martinez et~al.(2020)Martinez, Bertran, and
  Sapiro]{martinez2020minimax}
N.~Martinez, M.~Bertran, and G.~Sapiro.
\newblock Minimax pareto fairness: A multi objective perspective.
\newblock In \emph{International Conference on Machine Learning (ICML)}, 2020.

\bibitem[Martinez et~al.(2021)Martinez, Bertran, Papadaki, Rodrigues, and
  Sapiro]{martinez2021blind}
N.~L. Martinez, M.~A. Bertran, A.~Papadaki, M.~Rodrigues, and G.~Sapiro.
\newblock Blind {Pareto} fairness and subgroup robustness.
\newblock In \emph{International Conference on Machine Learning (ICML)}, 2021.

\bibitem[Menon et~al.(2019)Menon, Rawat, Reddi, and Kumar]{menon2019can}
A.~K. Menon, A.~S. Rawat, S.~J. Reddi, and S.~Kumar.
\newblock Can gradient clipping mitigate label noise?
\newblock In \emph{International Conference on Learning Representations}, 2019.

\bibitem[Mohri et~al.(2019)Mohri, Sivek, and Suresh]{mohri2019agnostic}
M.~Mohri, G.~Sivek, and A.~T. Suresh.
\newblock Agnostic federated learning.
\newblock In \emph{International Conference on Machine Learning (ICML)}, pages
  4615--4625. PMLR, 2019.

\bibitem[Nam et~al.(2020)Nam, Cha, Ahn, Lee, and Shin]{nam2020learning}
J.~Nam, H.~Cha, S.~Ahn, J.~Lee, and J.~Shin.
\newblock Learning from failure: De-biasing classifier from biased classifier.
\newblock In \emph{Advances in Neural Information Processing Systems
  (NeurIPS)}, volume~33, pages 20673--20684, 2020.

\bibitem[Nam et~al.(2021)Nam, Kim, Lee, and Shin]{iclr_2022_SSA}
J.~Nam, J.~Kim, J.~Lee, and J.~Shin.
\newblock Spread spurious attribute: Improving worst-group accuracy with
  spurious attribute estimation.
\newblock In \emph{International Conference on Learning Representations
  (ICLR)}, 2021.

\bibitem[Oakden-Rayner et~al.(2020)Oakden-Rayner, Dunnmon, Carneiro, and
  R{\'e}]{oakden2020hidden}
L.~Oakden-Rayner, J.~Dunnmon, G.~Carneiro, and C.~R{\'e}.
\newblock Hidden stratification causes clinically meaningful failures in
  machine learning for medical imaging.
\newblock In \emph{Proceedings of the ACM conference on health, inference, and
  learning}, pages 151--159, 2020.

\bibitem[Rosenberg et~al.(2005)Rosenberg, Hebert, and
  Schneiderman]{rosenberg2005semi}
C.~Rosenberg, M.~Hebert, and H.~Schneiderman.
\newblock Semi-supervised self-training of object detection models.
\newblock 2005.

\bibitem[Sagawa et~al.(2019)Sagawa, Koh, Hashimoto, and
  Liang]{sagawa2019distributionally}
S.~Sagawa, P.~W. Koh, T.~B. Hashimoto, and P.~Liang.
\newblock Distributionally robust neural networks for group shifts: On the
  importance of regularization for worst-case generalization.
\newblock In \emph{International Conference on Learning Representations}, 2019.

\bibitem[Sohn et~al.(2020)Sohn, Berthelot, Li, Zhang, Carlini, Cubuk, Kurakin,
  Zhang, and Raffel]{sohn2020fixmatch}
K.~Sohn, D.~Berthelot, C.-L. Li, Z.~Zhang, N.~Carlini, E.~D. Cubuk, A.~Kurakin,
  H.~Zhang, and C.~Raffel.
\newblock {FixMatch}: Simplifying semi-supervised learning with consistency and
  confidence.
\newblock In \emph{NeurIPS}, 2020.

\bibitem[Sohoni et~al.(2020)Sohoni, Dunnmon, Angus, Gu, and
  R\'{e}]{sohoni2020no}
N.~Sohoni, J.~Dunnmon, G.~Angus, A.~Gu, and C.~R\'{e}.
\newblock No subclass left behind: Fine-grained robustness in coarse-grained
  classification problems.
\newblock In \emph{Advances in Neural Information Processing Systems
  (NeurIPS)}, volume~33, pages 19339--19352, 2020.

\bibitem[Utama et~al.(2021)Utama, Moosavi, and Gurevych]{utama2020towards}
P.~A. Utama, N.~S. Moosavi, and I.~Gurevych.
\newblock Towards debiasing {NLU} models from unknown biases.
\newblock In \emph{Empirical Methods in Natural Language Processing (EMNLP)},
  2021.

\bibitem[Wah et~al.(2011)Wah, Branson, Welinder, Perona, and
  Belongie]{WahCUB_200_2011}
C.~Wah, S.~Branson, P.~Welinder, P.~Perona, and S.~Belongie.
\newblock {The Caltech-UCSD Birds-200-2011 Dataset}.
\newblock Technical Report CNS-TR-2011-001, California Institute of Technology,
  2011.

\bibitem[Wang et~al.(2020)Wang, Guo, Narasimhan, Cotter, Gupta, and
  Jordan]{wang2020robust}
S.~Wang, W.~Guo, H.~Narasimhan, A.~Cotter, M.~Gupta, and M.~I. Jordan.
\newblock Robust optimization for fairness with noisy protected groups.
\newblock In \emph{Advances in Neural Information Processing Systems
  (NeurIPS)}, 2020.

\bibitem[Xie et~al.(2020{\natexlab{a}})Xie, Dai, Hovy, Luong, and
  Le]{xie2019unsupervised}
Q.~Xie, Z.~Dai, E.~Hovy, M.-T. Luong, and Q.~V. Le.
\newblock Unsupervised data augmentation for consistency training.
\newblock In \emph{Advances in Neural Information Processing Systems
  (NeurIPS)}, 2020{\natexlab{a}}.

\bibitem[Xie et~al.(2020{\natexlab{b}})Xie, Luong, Hovy, and Le]{xie2020self}
Q.~Xie, M.-T. Luong, E.~Hovy, and Q.~V. Le.
\newblock Self-training with noisy student improves {ImageNet} classification.
\newblock In \emph{Proceedings of the IEEE/CVF Conference on Computer Vision
  and Pattern Recognition (CVPR)}, pages 10687--10698, 2020{\natexlab{b}}.

\bibitem[Yaghoobzadeh et~al.(2021)Yaghoobzadeh, Mehri, Tachet, Hazen, and
  Sordoni]{yaghoobzadeh2019increasing}
Y.~Yaghoobzadeh, S.~Mehri, R.~Tachet, T.~J. Hazen, and A.~Sordoni.
\newblock Increasing robustness to spurious correlations using forgettable
  examples.
\newblock In \emph{Conference of the European Chapter of the Association for
  Computational Linguistics (ACL)}, 2021.

\bibitem[Zhang et~al.(2021)Zhang, Menon, Veit, Bhojanapalli, Kumar, and
  Sra]{zhang2020coping}
J.~Zhang, A.~Menon, A.~Veit, S.~Bhojanapalli, S.~Kumar, and S.~Sra.
\newblock Coping with label shift via distributionally robust optimisation.
\newblock In \emph{International Conference on Learning Representations
  (ICLR)}, 2021.

\bibitem[Zhang et~al.(2022)Zhang, Sohoni, Zhang, Finn, and
  R\'{e}]{zhang2022correctncontrast}
M.~Zhang, N.~Sohoni, H.~Zhang, C.~Finn, and C.~R\'{e}.
\newblock {Correct-N-Contrast}: a contrastive approach for improving robustness
  to spurious correlations.
\newblock \emph{arXiv preprint arXiv:2203.01517}, 2022.

\bibitem[Zoph et~al.(2020)Zoph, Ghiasi, Lin, Cui, Liu, Cubuk, and
  Le]{zoph2020rethinking}
B.~Zoph, G.~Ghiasi, T.-Y. Lin, Y.~Cui, H.~Liu, E.~D. Cubuk, and Q.~V. Le.
\newblock Rethinking pre-training and self-training.
\newblock In \emph{Advances in Neural Information Processing Systems
  (NeurIPS)}, 2020.

\end{thebibliography}

\clearpage
\onecolumn
\setcounter{section}{0}
\renewcommand{\thesection}{\Alph{section}}

\section{Theoretical Proofs}
\label{app:theory}

\subsection{Notation}
The notation ``$i \in [b]$'' means $i \in \{1,2,\dots,b\}$. $\log$ denotes the natural logarithm. Where unspecified, $\norm{\cdot}$ denotes the Euclidean norm $\norm{\cdot}_2$. $\Delta^{d}$ denotes the $d$-coordinate ($d-1$ dimensional) simplex.

We have a training dataset $\{(x_i,y_i,z_i)\}_{i=1}^n \in \mathcal{X} \times \mathcal{Y} \times \mathcal{G}$. Here $x_i$ is a datapoint, $y_i$ is a discrete class label, and $z_i$ is a discrete group label.
Let $\mathcal{F}_1$ be a function class $\{f_\T : \X\times \Y \rightarrow \Delta^{|\G|}\}$, where each $f_\T$ is a member of $\mathcal{F}_1$ parameterized by the vector $\theta \in \R^d$. (The inputs to $f_\T$ are the features $x$ and class label $y$, and the output is a vector of predicted probabilities for each group.) Define $\mathcal{F}^R_1 := \{f_\T \in \mathcal{F}_1: \norm{\T}_2 \le R\}$.
Similarly let $\mathcal{F}_2$ be a function class $\{f_\T : \X \rightarrow \Delta^{|\Y|}\}$ and $\mathcal{F}^R_2 := \{f_\T \in \mathcal{F}_2: \norm{\T}_2 \le R\}$.
Let $\ell(\cdot, \cdot)$ be a nonnegative loss function globally bounded by $B$; $\ell$ is either in $\R^{|\Y|} \times \mathcal{Y} \ra \R$ or $\R^{|\G|} \times \mathcal{G} \ra \R$ depending on whether it is for the group classification task or target classification task, and will be clear from context. We will use the notations $\ell(f_\T(x), y)$ and $\ell(x, y; \theta)$ interchangeably (or even $\ell(f(x), y)$ or $\ell(x,y)$ when clear from context).

We assume that the datapoints are IID samples from a distribution $P$, i.e., $(x_i,y_i,z_i) \sim P$. We can write $P$ as a mixture of distributions $p_1,\dots,p_{|\mathcal{G}|}$, where $p_g$ is the distribution conditioned on the group label being equal to $g$. By overloading notation, we will also write $\E_{(x,y) \sim p_g}$ to denote the expectation of a quantity conditioned on the group label being $g$ (i.e., we will not write $(x,y,z) \sim p_g$ because $z = g$ always for points sampled from $p_g$). Let $q_1,\dots,q_{|\mathcal{G}|}$ be the corresponding mixture weights, i.e., $q_g = P(z = g)$, and let $q := \min\limits_{g \in \mathcal{G}} q_g$.

Recall that we denote $\D_1 := \{(x_i,y_i,z_i)\}_{i=1}^m$ to be the group-labeled dataset of $m$ points, and $\D_2 := \{(x_i,y_i)\}_{i=m+1}^n$ to be the dataset of $n-m$ group-unlabeled points. We are interested in the regime where $m$ is small compared to $n$, so we shall implicitly assume that $n-m$ is $\Omega(n)$. In our theoretical results we assume for simplicity that $\D_1$ and $\D_2$ are samples from the same distribution (although the $z_i$'s are unobserved on $\D_2$). However, our analysis easily extends to the case where $\D_1,\D_2$ are sampled from distributions on which the group proportions $q_1, \dots, q_{|\mathcal{\G}|}$ differ, but the per-group distributions $p_g = P(x,y | z = g)$ are the same. (This only complicates notation, as we need to distinguish the different $q = \min_{g \in \G} q_g$ between $\D_1$ and $\D_2$.)

We will use the notation $\L_{robust}(\theta)$ to denote either $\max_{g \in \G} \E[\ell((x, y), z; \theta) | z = g]$ (worst-group population loss for the task of classifying the \emph{groups}) or $\max_{g \in \G} \E[\ell(x, y; \theta) | z = g]$ (worst-group population loss for the task of classifying the \emph{classes}), which will be clear in context. We will also use $\L_{robust}^*$ as shorthand for $\min\limits_{\theta} \L_{robust}(\theta)$.

\subsection{``Helper" Results}
We will use the following standard result from learning theory \citep{229t}:

\begin{theorem}
\label{thm:generalization}
Suppose $\ell(x, y; \theta)$ is nonnegative, globally bounded by $B$ and $L$-Lipschitz continuous. Define $L(\theta) := \E_{(x,y) \sim P}[\ell(x, y; \theta)]$ and $\hat{L}(\theta) := \f{1}{n} \su{i=1}{n}\ell(x_i, y_i; \theta)$, where $\{(x_i,y_i)\}_{i=1}^n$ are sampled IID from distribution $P$. Let $p > 0$. Then, with probability $\ge 1 - O(e^{-p})$, for all $\theta$ such that $\norm{\theta}_2 \le R$ we have $|\hat{L}(\theta) - L(\theta)| \le O \lt B\sqrt{\f{p \max (\log (LRn), 1)}{n}} \rt$.
\end{theorem}

We will also use the following simple lemma relating the minimizer of a ``perturbed'' GDRO-style objective to the minimizer of the unperturbed version.
\begin{lemma}
\label{lem:perturb}
Define $f(\theta) = \max\limits_{k \in [G]} f_k(\theta)$ and $\tilde{f}(\theta) = \max\limits_{k \in [G]} \lt f_k(\theta) + \ep_k \rt$ where $|\ep_k| \le \ep$. Let the minimizers of $f,\tilde{f}$ be $\theta^*, \tilde{\theta}^*$ respectively. Then ${f}(\tilde{\theta}^*) - {f}({\theta}^*) \le 2\ep$.
\begin{proof}
Note that $\tilde{f}(\theta^*) \le f(\theta^*) + \ep$. Similarly $f(\tilde{\theta}^*) \le \tilde{f}(\tilde{\theta}^*) + \ep$. So $f(\tilde{\theta}^*) \le \tilde{f}(\tilde{\theta}^*) + \ep \le \tilde{f}({\theta}^*) + \ep \le {f}({\theta}^*) + 2\ep$. \end{proof}
\end{lemma}

\subsection{Warm-up: Proof of Group DRO Generalization Bound (Lemma \ref{lem:gdro})}
\label{app:gdro_proof}
We restate Lemma \ref{lem:gdro} below:
\gdro*

\emph{Proof}. 
The population group DRO loss is $\L_{robust}(\theta) :=  \max\limits_{g \in \mathcal{G}} \lt \E_{(x,y) \sim p_g} [\ell(x, y; \theta)]\rt$, i.e. the maximum of the average per-group losses. Here $p_g$ denotes the conditional distribution $P(x, y | z = g)$.
For $g \in \G$, denote $S_g$ to be the subset of points on $\D_1$ such that the group label is $g$. The empirical GDRO loss on $\D_1$ is $\L_{robust,\D_1}(\theta) := \max\limits_{g \in [G]} \lt \E_{(x,y) \sim S_g} [\ell(x, y; \theta)]\rt$. Note that each of the per-group losses is simply the empirical estimate of the corresponding population per-group loss, over the set $S_g$.

We can apply Theorem \ref{thm:generalization} to each of the groups individually, since each set $S_g$ is an IID sample from $p_g$. Thus, we obtain that for a given group $g$, $\left|\E_{(x,y) \sim S_g} [\ell(x, y; \theta)] - \E_{(x,y) \sim p_g} [\ell(x, y; \theta)]\right| \le O(\sqrt{\ff{p \log |S_g|}{|S_g|}})$ for \emph{all} $\theta$ such that $\norm{\theta}_2 \le R$, with probability $\ge 1-O(e^{-p})$. ($|S_g|$ is the size of set $S_g$.) Thus by union bound, $\max\limits_{g \in \G} \left|\E_{(x,y) \sim S_g} [\ell(x, y; \theta)] - \E_{(x,y) \sim p_g} [\ell(x, y; \theta)]\right| \le O\!\lt \!\max\limits_{S_g}\! \sqrt{\ff{p \log |S_g|}{|S_g|}}\rt$ w.p. $\ge 1-O(|\mathcal{G}|e^{-p})$.
(Here and henceforth we omit the constants $B,L,R$ from the notation.)

Let $q_g := P(z = g)$, i.e., the population fraction of group $g$. By Hoeffding's inequality, for a given group $g$ we have $|S_g-q_g m| \le \sqrt{m \log m}$ with probability $\ge 1 - 2/m^2$. Thus by union bound and the fact that $q_g \ge q$, we have 
\begin{align*}
   \min\limits_{g \in \G} |S_g| \ge q m - \sqrt{m \log m} = (1 - \ff{\log m}{q\sqrt{m}})qm 
\end{align*} with probability greater than $1 - 2|\G| / m^2$. Thus $\max\limits_{g \in \G} \ff{\log |S_g|}{|S_g|} \le \f{\log m}{qm} \cdot \f{1}{1 - \ff{\log m}{q\sqrt{m}}}$ with probability greater than $1 - 2|\G| / m^2$. As $\ff{1}{1-x} \le 1+2x$ for all $x \in [0, 1/2]$, we thus have
\begin{align*}
    \max\limits_{g \in \G} \f{\log |S_g|}{|S_g|} \le \min\left\{1, \f{\log m}{qm} + 2\f{\log m}{qm}  \cdot \f{\log m}{q\sqrt{m}}\right\}
\end{align*} 
for sufficiently large $m$ (relative to $q$). Thus \begin{align*}
    \max_g \sqrt{\ff{p \log |S_g|}{|S_g|}}= O\Bigg(\sqrt{\ff{\log m}{qm}}\Bigg)
\end{align*} with probability greater than $1-2|\G|/m^2$.

To recap, the training loss is $\max_{g \in \mathcal{G}} \E_{(x,y) \sim S_g} \E[\ell(x,y;\theta)]$ and we showed that 
\begin{align*}
    |\E_{(x,y) \sim S_g} \E[\ell(x,y;\theta)] - \E_{(x,y) \sim p_g} \E[\ell(x,y;\theta)]| \le  O\Bigg(\sqrt{\ff{\log m}{qm}}\Bigg)
\end{align*} for all $g \in \G$ and all $\theta$ with $\norm{\theta}_2 \le R$ with high probability, so applying Lemma~\ref{lem:perturb} and union bound, we have that 
\begin{align*}
    \L_{robust}(\hat{\theta}) - \L_{robust}^* = |\L_{robust}(\hat{\theta}) - \L_{robust}^*| \le  O\Bigg(\sqrt{\ff{\log m}{qm}}\Bigg) = \tilde{O}(\ff{1}{\sqrt{qm}})
\end{align*}
with high probability, as desired. Note that as long as we optimize over a bounded domain that is sufficiently large to contain $\theta^*$, there exists a valid norm constraint $R$ such that all $\theta$ under consideration have $\norm{\theta}_2 \le R$.

\subsubsection{Relating loss to classification error}
\label{app:losses}
For example, we could let $\ell$ be the truncated cross-entropy loss (i.e., the cross-entropy loss clipped to remain in $[0,B]$ for some large constant $B$ in order to ensure boundedness). Observe that the classification error is upper bounded by the cross-entropy loss divided by $\log(|\mathcal{Y}|)$, where $|\mathcal{Y}|$ is the number of classes: as the class prediction is the class with highest predicted probability, if the model makes an error  then the predicted probability of the correct class is at most $1/|\Y|$, which means that the cross-entropy loss for that example is $\ge -\log(1/\mathcal{\Y}) =  \log |\Y|$. In other words, if $\argmax \{f(x)\} \ne y$ then $\l(f(x), y) \ge \log |\Y|$. Thus $\mathbf{1}(\argmax \{f(x)\} \ne y) \le \ff{\l(f(x), y)}{\log |\Y|}$, and so $\E[\mathbf{1}(\argmax \{f(x)\} \ne y)] \le \ff{1}{\log |\Y|}\E[\l(f(x), y)]$. (Thus, the clipping constant $B$ just needs to be $\ge \log |\Y|$.)

We could also let $\ell$ be the squared loss (the square of 1 minus the predicted probability of the correct class), which is 2-Lipschitz and bounded by 1. The classification error is upper bounded by $\ff{|\Y|}{|\Y|-1}$ times the squared loss (since when the model makes an error, the loss on that example must be at least $(1-1/|\Y|)^2 + (|\Y|-1)/|\Y|^2 = 1 - 2/|\Y| + 1/|\Y| = \ff{|\Y|-1}{|\Y|}$).

\subsection{Proof of Theorem~\ref{thm:basic}} 
We restate Theorem~\ref{thm:basic} below, and prove it in Appendix \ref{app:pg}-\ref{app:eg}.
\vspace{5pt}
\basic*

\subsubsection{Per-group population loss vs. per-estimated-group population loss}
\label{app:pg}
Let $\tilde{z}$ denote the prediction of the Stage 1 group classifier, i.e., $\tilde{z} = \argmax\{\hat{f}_{group}(x, y)\}$. The assumption on the error rate on each group is equivalent to assuming $P(\tilde{z} \ne z | z = g) \le r$ for each $g \in \G$. Notice that $P(z \ne \tilde{z} | \tilde{z} = g) = \sum_{g' \in |\G|, g' \ne g}{} {P(z = g' | \tilde{z} = g)} = \sum_{g' \in |\G|, g' \ne g}{} \f{P(\tilde{z} = g | z = g') P(z = g')}{P(\tilde{z} = g)}$. Now, $P(\tilde{z}  = g) \ge P(\tilde{z} = g | z = g)P(z = g) \ge (1-r)q$, and ${P(\tilde{z} = g | z = g')} \le r$ for $g \ne g'$, and $\sum_{g' \ne g}{} P(z = g') \le 1-q$, so $P(z \ne \tilde{z} | \tilde{z} = g) \le \f{(1-q)r}{(1-r)q} \le \f{r}{(1-r)q}$.  Also, of course $P(z \ne \tilde{z} | \tilde{z} = g) \le 1$ as well, so $P(z \ne \tilde{z} | \tilde{z} = g) \le \min\{\ff{r}{q(1-r)}, 1\}$.

\begin{align*}
|\E[\l(x,y) | \tilde{z} = g] - \E[\l(x,y) | {z} = g]| &= \\
|\E[\l(x, y) | \tilde{z} = g, z = g]P(z = g | \tilde{z} = g) + \E[\l(x, y) | \tilde{z} = g, z \ne g]P(z \ne g | \tilde{z} = g) - \E[\l(x,y) | {z} = g]| &\le \\
|\E[\l(x, y) | \tilde{z} = g, z = g]P(z = g | \tilde{z} = g) - \E[\l(x,y) | {z} = g]| +  \E[\l(x, y) | \tilde{z} = g, z \ne g]P(z \ne g | \tilde{z} = g) &\le \\
|\E[\l(x, y) | \tilde{z} = g, z = g]P(z = g | \tilde{z} = g) - \E[\l(x,y) | {z} = g]| +  B\min\{\ff{r}{q(1-r)}, 1\}.
\end{align*}

\noindent Similarly
$\E[\l(x,y) | {z} = g] =
\E[\l(x, y) | \tilde{z} = g, z = g]P(\tilde{z} = g | {z} = g) + \E[\l(x, y) | \tilde{z} \ne g, z = g]P(\tilde{z} \ne g | {z} = g)$, so
\\
$
|\E[\l(x, y) | \tilde{z} = g, z = g]P(z = g | \tilde{z} = g) - \E[\l(x,y) | {z} = g]| = \\
{|\E[\l(x, y) | \tilde{z} = g, z = g]P(z = g | \tilde{z} = g) - \E[\l(x, y) | \tilde{z} = g, z = g]P(\tilde{z} = g | {z} = g) - \E[\l(x, y) | {z} \ne g, z = g]P(\tilde{z} \ne g | {z} = g)|} \\ \le
\E[\l(x, y) | \tilde{z} = g, z = g]|P(z = g | \tilde{z} = g)-P(\tilde{z} = g | {z} = g)| + Br.
$

Finally, we have $1 \ge P(z = g | \tilde{z} = g) = 1 - P(z \ne \tilde{z} | \tilde{z} = g) \ge \max\{1-\ff{r}{q(1-r)}, 0\}$, and $1 \ge P(\tilde{z} = g | z = g) \ge 1-r$, so $|P(z = g | \tilde{z} = g)-P(\tilde{z} = g | {z} = g)| \le \min\{\ff{r}{q(1-r)}, 1\}$. Thus, altogether we have \begin{align*}
    |\E[\l(x,y) | \tilde{z} = g] - \E[\l(x,y) | {z} = g]|  \le B\min\{\ff{r}{q(1-r)}, 1\} + B\min\{\ff{r}{q(1-r)}, 1\} + Br \le 
3B\min\{\ff{r}{q(1-r)}, 1\},
\end{align*}
which is $O(r/q)$ in terms of $r$ and $q$.

\subsubsection{Per-estimated-group population loss vs. training loss}
\label{app:eg}
Let $\tilde{S}_g$ denote the set of training points with predicted group label $g$. The datapoints in $\tilde{S}_g$ are independent samples from ${P(x, y | \tilde{z} = g)}$. So by Theorem~\ref{thm:generalization}, for a given $g$ we have that the difference between the per-group training and population losses on the Stage 2 classification task is $|\E_{(x,y) \sim \tilde{S}_g}[\ell(x, y; \theta)] - \E[\ell(x, y; \theta) | \tilde{z} = g]| \le O\lt \sqrt{\ff{\log |\tilde{S}_g|}{|\tilde{S}_g|}}\rt$ with high probability.

By Hoeffding's inequality we have $O\lt \sqrt{\ff{\log |\tilde{S}_g|}{|\tilde{S}_g|}}\rt \le O\lt \sqrt{\ff{\log n}{q(1-r)n}}\rt=O\lt \sqrt{\ff{\log n}{qn}}\rt$.
Thus by triangle inequality and union bound, we have 
$|\E_{(x,y) \sim \tilde{S}_g}[\ell(x, y; \theta)] - \E[\ell(x, y; \theta) | {z} = g]| \le O(r/q) + O\lt \sqrt{\ff{\log n}{qn}}\rt = \tilde{O}\lt \ff{r}{q} + \ff{1}{\sqrt{qn}}\rt$ for all $g \in \G$ and all $\theta$ with $\norm{\theta}_2 \le R$ with high probability. Finally, applying Lemma~\ref{lem:perturb} yields the desired result, as the training loss is $\max\limits_{g \in \G} \E_{(x,y) \sim \tilde{S}_g}[\l(x, y; \theta)]$. \qquad $\square$

\subsection{Extension of Theorem \ref{thm:basic}: Non-realizable case}
In this section, we show that if the group classifier makes randomized predictions according to the predicted probabilities (rather than classifying groups by picking the group with maximum predicted probability), we can yield a more general bound than Theorem~\ref{thm:basic} that is better in the non-realizable case (i.e., when there does not exist a perfect group classifier). We relax the realizability requirement stated in Section~\ref{sec:analysis} to the  requirement that there must exist some $f^* \in \mathcal{F}_1$ such that $f(x,y)$ is the vector of true probabilities $P(z | x,y)$. This theorem is stated below.

\begin{theorem}
\label{thm:bayes}
Suppose there exists $f^* \in \mathcal{F}_1$ such that $f(x,y)$ is the vector of true probabilities $P(z | x,y)$. 
Let $\hat{P}(z | x,y)$ denote the probabilities output by the group classifier from "Stage 1," and suppose for each datapoint $(x,y)$ the predicted group label $\tilde{z}$ is sampled from $\hat{P}(z | x,y)$. 
Suppose that the total variation between $\hat{P}(z | x,y)$ and $P(z | x,y)$ is bounded by $r$, for all $x,y$ in the support of $P$.
Then with high probability, $\L_{robust}(\hat{\theta}_{robust}) \le \L^*_{robust} + \tilde{O} \lt \ff{r}{q} + \ff{1}{\sqrt{qn}} \rt$.
\end{theorem}

Theorem~\ref{thm:bayes} shows that even if it is impossible to perfectly distinguish the groups, this is not necessarily an obstacle to the downstream robust performance.
To prove Theorem~\ref{thm:bayes}, we first prove the following lemma.

\begin{lemma}
\label{lem:help}
Suppose $P(\tilde{z}|x,y) = P(z | x,y)$. Then $\E[\l(x,y;\theta) | z = g] = \E[\l(x,y;\theta) | \tilde{z} = g]$, for all $g \in \mathcal{G}$.
\end{lemma}
\begin{proof}
By Bayes' rule $P(x,y | \tilde{z} = g) = P(\tilde{z} = g | x,y) P(x,y) / P(\tilde{z} = g)$. By assumption, for any $(x,y)$, $P(\tilde{z} = g | x,y) = P(z = g | x,y)$. Therefore $P(x,y | \tilde{z} = g) = P(\tilde{z} = g | x,y) P(x,y) / P(\tilde{z} = g) = P(z = g | x,y) P(x,y) / P(z = g) = P(x,y | z = g)$ by applying Bayes' rule again. The claim follows.
\end{proof}

Lemma \ref{lem:help} implies that if we use the predicted group labels $\tilde{z}$ rather than the ``true'' group labels $z$, there is essentially no difference since $\tilde{z}$ and $z$ have the same distribution conditioned on $x,y$. Thus, samples from $(x,y,\tilde{z})$ and $(x,y,z)$ are equivalent for our purposes, and applying Lemma~\ref{lem:gdro} shows that the minimizer of $\hat{\tilde{\theta}}$ of $\max\limits_g \E[\l(x,y;\theta) | \tilde{z} = g]$ satisfies $\L_{robust}(\hat{\tilde{\theta}}) \le \L_{robust}^* + \tilde{O}(\ff{1}{\sqrt{qn}})$ with high probability (since in this case we have $n$ total datapoints).

\noindent \emph{Proof of Theorem~\ref{thm:bayes}}. Note that if $r = 1$ the statement follows trivially from boundedness. Similarly if $r = 0$ the statement follows from the argument above. So assume $0 < r < 1$.

Given $x,y$, suppose we sample $z',\tilde{z}'$ in the following ``coupled'' fashion.
Flip a biased coin with probability of heads being $\sum_{g \in \G} \min\big\{\hat{P}(z = g | x,y), P(\tilde{z} = g | x,y)\big\}$.

If the coin is heads, then sample $z''$ from the distribution where \begin{align*}
    P(z''=g) = \f{\min\{\hat{P}(z = g | x,y), P(\tilde{z} = g | x,y)\}}{\sum_{g \in \G} \min\{\hat{P}(z = g | x,y), P(\tilde{z} = g | x,y)\}}.
\end{align*} 
and set $z' = \tilde{z}' = z''$.
Note that by the assumption $r < 1$, the denominator is nonzero.

If the coin is tails, then sample $z'$ from the distribution where
\begin{align*}
P(z'=g) = \f{P(z = g | x,y) - \min\{\hat{P}(z = g | x,y), P(z = g | x,y)\}}{1-\sum_{g \in \G} \min\{\hat{P}(z = g | x,y), P(\tilde{z} = g | x,y)\}},    
\end{align*}
and independently sample $\tilde{z}'$ from the distribution where 
\begin{align*}
{P(\tilde{z}'=g)} = \f{P(\tilde{z} = g | x,y) - \min\{\hat{P}(z = g | x,y), P(\tilde{z} = g | x,y)\}}{1-\sum_{g \in \G} \min\{\hat{P}(z = g | x,y), P(\tilde{z} = g | x,y)\}}.
\end{align*}
Also note that by the assumption $r > 0$, the denominator is nonzero. 

Notice that using this sampling procedure, $P(\tilde{z}' | x,y) = P(z | x,y)$. Similarly, $P(z' | x,y) = \hat{P}(z' | x,y)$. Thus by Lemma~\ref{lem:help}, $\E[\l(x,y;\theta)|z' = g] = \E[\l(x,y;\theta)|z= g] = \L_{robust}(\theta)$, and $\E[\l(x,y;\theta)|\tilde{z}' = g] = \E[\l(x,y;\theta)|\tilde{z}= g]$. Also observe that the probability that $z' \ne \tilde{z}'$ is $\le 1-\sum_{g \in \G} \min\{\hat{P}(z = g | x,y), P(\tilde{z} = g | x,y)\}$ which is precisely the total variation between $\hat{P}(z = g | x,y)$ and $P(\tilde{z} = g | x,y)$. From here, the remainder of the proof is essentially identical to the proof of Theorem~\ref{thm:basic}, as the sampled group labels $z'$ from the true conditional distribution are ``equivalent'' to ``correct'' group labels, and the disagreement rate between $z'$ and the estimated group labels $\tilde{z}'$ is $\le r$ by the total variation assumption.

\subsection{Proof of Corollary~\ref{coro:base}}
\base*

\emph{Proof}. First, by Lemma~\ref{lem:gdro}, we have that if we train a classifier with group DRO on $\D_1$ to classify the \emph{group labels}, the worst-group population loss is $\tilde{O}(1/\sqrt{qm})$ with high probability (since in the realizable case, there exists a group classifier with 0 population loss). As discussed in Appendix \ref{app:losses}, for both the cross-entropy and the squared loss this translates to a (population) worst-group misclassification \emph{error} of $\tilde{O}(1/\sqrt{qm})$ as well (when the classifier prediction is the group with maximum predicted probability).

In fact, by inspecting the proof of Lemma~\ref{lem:gdro}, we can make the slightly stronger statement that the population classification error on group $g$ is $\tilde{O}(1/\sqrt{q_g m})$ with high probability.

So, $P(z \ne g | \tilde{z} = g) = \su{g' \in |\G|, g' \ne g}{} \f{P(\tilde{z} = g | z = g') P(z = g')}{P(\tilde{z} = g)} = \su{g' \in |\G|, g' \ne g}{} \f{\tilde{O}(1/\sqrt{q_{g'} m}) \cdot q_{g'}}{P(\tilde{z} = g)} \le O \lt \f{1}{(1-\tilde{O}(1/\sqrt{q_g m}))q_g\sqrt{m}} \sum \sqrt{q_{g'}}\rt = \tilde{O}(\ff{1}{q\sqrt{m}})$
since $P(\tilde{z} = g) \ge P(\tilde{z} = g | z = g)P(z = g) \ge (1-\tilde{O}(1/\sqrt{q_g m}))q_g$ with high probability and $\sum\sqrt{q_g'} \le \sum 1/\sqrt{|\G|} = \sqrt{|\G|}$.
Thus by a similar argument to the proof of Theorem~\ref{thm:basic}, we have $|\E[\l(x,y) | \tilde{z} = g] - \E[\l(x,y) | {z} = g]| \le  \tilde{O}(\ff{1}{q\sqrt{m}})$, and in turn can obtain the desired result (by an analogous argument as Appendix \ref{app:eg}).

\subsection{Proof of Corollary~\ref{coro:ssl}}
Corollary~\ref{coro:ssl} is restated below; recall that it assumes that the group classification problem is realizable (i.e., there exists a classifier with 0 training loss on the group classification task) and that $m = \Omega(\sqrt{n})$.

\ssl*

\noindent To prove Corollary~\ref{coro:ssl}, we make use of the following theorem (see Theorem~\ref{thm:ssl_generic} below, a restated version of Theorem 21.8 from \citep{balcan2009discriminative}). First, we need some additional notation: let $\chi : \mathcal{F}_1 \times \X\ra \{0, 1\}$ be a function and define the overloaded notation $\chi(f, P) = \E_{x \sim P}[\chi(f, x)]$. Let $VC(\mathcal{C})$ denote the VC-dimension of a function class $\mathcal{C}$. For function classes that output probabilities rather than labels directly, we overload notation so that $VC(\mathcal{F}) = VC(\{\argmax\{f\} \,|\, f \in \F\})$. Recall that $\mathcal{D}_2$ denotes the training set of group-unlabeled points. 

\begin{theorem}
\label{thm:ssl_generic}
Given 
$\de \in (0, 1)$, if
$n = \Omega \lt \ff{1}{\ep^2} \lt VC(\chi(\F_1)) \log \ff{1}{\ep} + \log \ff{1}{\de} \rt\rt$ and
$m \ge \ff{2}{\ep}\lt \log(2s(2m,t+2\ep)) + \log \ff{2}{\de}\rt$,
then with probability $\ge 1-\de$ it holds that all $f \in \mathcal{F}_1$ with zero training error and $1-\chi(f, \mathcal{D}_2) \le t+ \ep$
have population error $\le \ep$. Here, $s(2m,t+2\ep)$ denotes the expected number of splits when $2m$ points are drawn IID from $\P$ with concepts $f \in \F_1$ having $\chi(f,\P) \le t + 2\ep$.
\end{theorem}

\noindent \emph{Proof of Corollary~\ref{coro:ssl}}.
FixMatch \citep{sohn2020fixmatch} (which we use for Stage 1 of \name-SSL, as described in \ref{alg:bssl}) minimizes a weighted sum of the \emph{supervised} loss (computed on the training points with group labels) and an \emph{unsupervised} loss (consistency of predictions between examples and augmented versions of the same example, computed on all training points). Concretely, in our case the FixMatch loss for a group classifier $f \in \F_1$ is ${\L}_{fixmatch} = \M \L_{sup} + (1-\M) \L_{consistency}$, where $\M \in [0, 1]$, $\L_{sup} = \E_{(x,y,z) \sim \D_1}[\ell(f(x,y), z)]$, and $\L_{consistency} = \E_{(x,y) \sim \D_2}[\mathbf{1}(\max\{f(x,y)\} \ge \tau) \cdot \ell(f(aug(x), y), f(x,y))]$, where $\tau \in [0, 1]$ is a predefined constant.

Here, $aug(\cdot)$ denotes the augmentation function; for simplicity assume it is a fixed (non-random) function. Let $\chi(f, x)$ be 0 if $\max\{ f(x)\} \ge \tau$ and $\argmax \{f(x)\} \ne \argmax \{f(aug(x))\}$, and 1 otherwise (in other words, $\chi(f,x)$ is 1 unless $f$ makes a confident prediction on $x$ but makes a different prediction on the augmented version of $x$). By the realizability assumption, there exists $f \in \F_1$ with zero loss (and therefore zero training loss) on the supervised task of classifying the groups. Note that if $aug$ is the identity function, then $f$ also attains zero consistency loss; more generally, given $\xi \ge 0$ we can choose $aug$ to be a ``weak enough'' augmentation such that there exists $f \in \F_1$ such that $\L_{consistency} \le \xi$. If we choose $\xi$ to be sufficiently small, we can guarantee that $\L_{consistency} \le \delta$ implies $\chi(f, \D_2) \le t + \ep$, since $t,\ep \ge 0$.

Because there exists $f \in \F_1$ with zero loss, this implies that if we set the weight $\M$ on the supervised part of the loss to be large enough, the function $\F_1$ we learn (corresponding to the minimizer of $\L_{fixmatch}$) will have zero training error. (If we find a function with nonzero training error, we can increase the weight of the supervised loss and rerun.) 
Then, invoking Theorem~\ref{thm:ssl_generic}, we get that the \emph{population} error of the returned group classification model will be $\le O(\ep)$ with high probability as long as $m,n$ satisfy $m = \Omega(1/\ep^2)$ and $n = \Omega(1/\ep)$. 
The value of the term $t$ in Theorem~\ref{thm:ssl_generic} depends on the consistency error $\chi(f, \D_2)$ on the training data, which in turn will depend on the augmentation chosen, the threshold $\tau$, and the weight $1-\M$ (smaller $\M$ will encourage lower consistency loss). 
Similarly, the constants in the preceding $\Omega(\cdot)$ terms depend on the choice of augmentation function $aug(\cdot)$; intuitively, we would like to choose an augmentation function that is as strong as possible while still being label-preserving (so as not to make the consistency loss large). 

In summary, given the required conditions, if $m = \Omega(1/\ep^2)$ and $n = \Omega(1/\ep)$ then the population error of the returned group classification model is $\le O(\ep)$ with high probability. The population error of the group classifier on each group $g$ will then be $\le O(\ep / q_g)$ with high probability. Equivalently, with high probability, for all groups $g$ the population error of the group classifier on that group will be $O(\ff{1}{q_gm}+ \ff{1}{q_g\sqrt{n}})$ [as long as $m = \Omega(\sqrt{n})$].
Finally, the desired result now follows by invoking Theorem~\ref{thm:basic}. 

\subsection{Proof of Corollary~\ref{coro:strong}}
Corollary~\ref{coro:strong} is restated below.
\strong*

\emph{Proof}.
In words, the assumption that $\ell(f(x), y) \perp \argmax\{\hat{f}_{group}(x, y)\}\, \big|\, z$ for all $f \in \mathcal{F}_2$ says that for all classifiers in $\mathcal{F}_2$, the loss on the \emph{target task} (of classifying the class labels) is independent of the prediction of the group classifier, when conditioned on the actual group label.

Using this assumption, we have that $\E_{(x,y)}[\l(x,y;\theta) | \tilde{z} = g] = \\
\E_{(x,y)}[\l(x,y;\theta) | \tilde{z} = g, z = g]P(z = g | \tilde{z} = g) +  \E_{(x,y)}[\l(x,y;\theta) | \tilde{z} = g, z \ne g] \cdot {P(z \ne g | \tilde{z} = g)} = \\
\E_{(x,y)}[\l(x,y;\theta) | z = g]P(z = g | \tilde{z} = g) +  \E_{(x,y)}[\l(x,y;\theta) | z \ne g]P(z \ne g | \tilde{z} = g)$.

Denote $\L_{avg}(\theta) = \E_{(x,y)}[\l(x,y;\theta)]$. Note that $\E_{(x,y)}[\l(x,y;\theta) | z \ne g]P(z \ne g) + \E_{(x,y)}[\l(x,y;\theta) | z = g]P(z = g) = L_{avg}(\theta)$. Thus, 
\begin{align*}
    \E_{(x,y)}[\l(x,y;\theta) | z \ne g]P(z \ne g | \tilde{z} = g) = \f{P(z \ne g | \tilde{z} = g)}{P(z \ne g)} \lt L_{avg}(\theta)- \E_{(x,y)}[\l(x,y;\theta) | z = g]P(z = g)\rt.
\end{align*}
As a result, 
\begin{align*}
    \E_{(x,y)}[\l(x,y;\theta) | \tilde{z} = g]  =
(1-a)\E_{(x,y)}[\l(x,y;\theta) | z = g] + a L_{avg}(\theta),
\end{align*}
where $a = \f{P(z \ne g | \tilde{z} = g)}{P(z \ne g)} \le \tilde{O}(\ff{1}{\sqrt{qm}})$ with high probability (as shown in previous sections).

Let $\tilde{\theta}_{pop}$ denote the minimizer of $\E_{(x,y)}[\l(x,y;\theta) | \tilde{z} = g]$.
Suppose for contradiction that $\L_{robust}(\tilde{\theta}_{pop}) > \L_{robust}(\theta^*_{avg})$.
Let 
\begin{align*}
    k_{max} \in \argmax_{k \in \G} \lt (1-a)\cdot\E_{(x,y)\sim p_k} [\l(x, y; \theta)] + a \cdot \E_{(x,y)\sim P}[\l(x,y;\theta)] \rt.
\end{align*}
Then we have $\lt (1-a)\cdot\E_{(x,y)\sim p_{k_{max}}} [\l(x, y; \theta)] + a \cdot \E_{(x,y)\sim P}[\l(x,y;\theta)] \rt > \E_{(x,y)\sim p_{k_{max}}} [\l(x, y; \theta)]$, i.e. that $\E_{(x,y)\sim P}[\l(x,y;\theta)] > \E_{(x,y)\sim p_{k_{max}}}[\l(x,y;\theta)] $ and thus that $\lt (1-a)\cdot\E_{(x,y)\sim p_{k_{max}}} [\l(x, y; \theta)] + a \cdot \E_{(x,y)\sim P}[\l(x,y;\theta)] \rt < \E_{(x,y)\sim P}[\l(x,y;\theta)]$. But there must be some $k \in [\G]$ such that $\lt (1-a)\cdot\E_{(x,y)\sim p_{k_{max}}} [\l(x, y; \theta)] + a \cdot \E_{(x,y)\sim P}[\l(x,y;\theta)] \rt \ge \E_{(x,y)\sim P}[\l(x,y;\theta)]$, which contradicts the definition of $k_{max}$.
Thus, $\L_{robust}(\tilde{\theta}_{pop}) \le \L_{robust}(\theta^*_{avg})$. 

Using the fact that the datapoints in $\tilde{S}_g$ are independent samples from ${P(x, y | \tilde{z} = g)}$, we have for all $g$ that $|\E_{(x,y) \sim \tilde{S}_g}[\ell(x, y; \theta)] - \E[\ell(x, y; \theta) | \tilde{z} = g]| \le O\lt \sqrt{\ff{\log |\tilde{S}_g|}{|\tilde{S}_g|}}\rt$ with high probability (as argued in Appendix \ref{app:eg}). By combining this with the result of Corollary~\ref{coro:base}, we obtain the desired result. $\square$

\subsubsection{Proof of Lemma~\ref{lem:erm}}
\erm*

\emph{Proof.} Consider the following simple distribution. Suppose that the distribution of $x$ is a point mass on a single point. Suppose that there are $k$ classes and the classes are identical to the groups (that is, $y = z$ always). If cross-entropy loss is used, then clearly to minimize worst-group loss one should predict a uniform distribution over each class (since the point $x$ gives no information about the class or group). Thus, the loss on each point would be $\log k$, so $\L_{robust}^* = \log k$ in this case.

Recall that the cross-entropy loss has the property that the \emph{average} population loss is minimized when the predicted probability of each class is simply the true probability of that class conditioned on the features $x$. In our case, the latter is simply the probability of the class. Thus, for the minimizer of the population average cross entropy loss, the predicted class probabilities for any point are the true class probabilities $P(c)$, and so the loss given that the true class is $c$ is $-\log P(c)$ by definition. So the worst-group loss is $\max_c -\log P(c) = -\log \min_c P(c) = \log(1/q)$ [note that $1/q \ge k$].  The lemma now follows by a simple application of Hoeffding's inequality, since the training dataset has $n$ points. 

\section{Experimental Details and Additional Results}
\label{app:experiments}
\subsection{Pseudocode}
We present Algorithm~\ref{alg:bssl} to explain $\name$-SSL in more detail. We also include algorithm boxes for the subroutines used in Algorithm~\ref{alg:bbase}.

\begin{figure}
\begin{algorithm}[H]
\caption{\name-SSL}
\label{alg:bssl}
\begin{algorithmic}[1]
\Require Group-labeled data $\mathcal{D}_1 = \{(x_i,y_i,z_i)\}_{i=1}^m$, group-unlabeled data $\mathcal{D}_2 = \{(x_i,y_i)\}_{i=m+1}^n$.
\State $\hat{f}_{group} \leftarrow$ {\sc Train\_FixMatch}$(\mathcal{D}_1,\mathcal{D}_2)$
\State $(\hat{z}_{m+1},\dots,\hat{z}_n) \leftarrow \textsc{Predict}(\hat{f}_{group},\mathcal{D}_2)$
\State $(\hat{z}_1,\dots,\hat{z}_m) \leftarrow (z_1,\dots,z_m)$
\State $\hat{f}_{robust} \leftarrow$  {\sc{GDRO}}$\lt \{(x_i,y_i,\hat{z}_i)\}_{i=1}^n \rt$
\Ensure Final model  $\hat{f}_{robust}$.
\end{algorithmic}
  \end{algorithm}
    \vspace{-2em}
  \caption{\name-SSL. A semi-supervised model is trained on both the group-labeled training examples and the training examples with only class labels, and generates group pseudolabels for the training data without group labels. These pseudolabels are used to train a more robust model using GDRO \citep{sagawa2019distributionally}.
  For FixMatch training, we use the standard FixMatch algorithm \citep{sohn2020fixmatch} except we modify it to take the class label $y$ as input as well, similar to how we train $\name$-\textsc{Base}.
}
\end{figure}

\begin{figure}
\begin{algorithm}[H]
\caption{\textsc{Train\_Supervised}}
\begin{algorithmic}[1]
\Require Group-labeled data $\mathcal{D}_1 = \{(x_i,y_i,z_i)\}_{i=1}^m$, group-labeled validation data $\mathcal{D}_{val}$
\State Initialize model $f_\T \in \mathcal{F}_1$
\State Initialize acc\_best = 0
\For {epoch in $1,\dots,T$}
\For {$i$ in $1,\dots,m$}
\State $\T \leftarrow \text{GDRO\_Update}(\ell, f_\theta(x_i, y_i), z_i)$
\EndFor
\If {ValidationSubsetWorstGroupAcc($f_\theta$) $>$ acc\_best}
\State $\hat{f}_{group} \leftarrow f_\theta$
\State acc\_best $\leftarrow$ ValidationSubsetWorstGroupAcc($f_\theta$)
\EndIf
\EndFor
\Ensure Group prediction model  $\hat{f}_{group}$.
\end{algorithmic}
  \end{algorithm}
    \vspace{-2em}
  \caption{Details for training the supervised \emph{group} classifier. In our experiments, it is typically initialized from a pretrained ResNet-50 model, except for MNIST where it is a randomly initialized LeNet. We minimize the GDRO training loss (in our experiments, $\ell$ is the cross-entropy loss), and select the model that does best on the validation subset with group labels, in terms of the worst-group accuracy on that subset. The group classifier takes both the features $x$ and class label $y$ as input. (For the specific equation of the GDRO update, see \citep{sagawa2019distributionally}. Note that in practice, we do a minibatched version of the above.)}
\end{figure}

\begin{figure}
\begin{algorithm}[H]
\caption{\textsc{Predict}}
\begin{algorithmic}[1]
\Require Group prediction model $\hat{f}_{group}$, group-unlabeled training data $\mathcal{D}_2 = \{(x_i,y_i)\}_{i=m+1}^n$.
\For {$i=1,\dots,m$}
\State $\hat{z}_i \leftarrow \argmax_{g \in \G} \hat{f}_{group}(x_i, y_i)$
\EndFor
\Ensure $\{\hat{z}_i\}_{i=1}^m$
\end{algorithmic}
  \end{algorithm}
    \vspace{-2em}
  \caption{Details for extracting group predictions from the trained group prediction model.}
\end{figure}
\subsection{Training Details}
\label{app:train-details}
For all datasets and methods, we use a fixed training/validation/test split. Models are trained on the training set, and the validation set is used for model selection (both selecting the best model during training, and for hyperparameter selection). All results reported in plots and tables are on the test set; the test set is \textbf{not} used for any model selection or tuning purposes. For consistency and direct comparability with prior works, we do not use data augmentation for any of the baselines or while training the ``Stage 2'' robust models, except when explicitly specified otherwise. For all experimental settings, reported means and standard deviations are over 5 trials with different random seeds.

\paragraph{U-MNIST.} This task is based on the MNIST dataset \citep{lecun2010mnist} (available under the Creative Commons Attribution-Share Alike 3.0 license). The U-MNIST task is to classify digits between $\le 4$ and $\ge 5$; the groups are the individual digits. We use a fixed training-validation split for all methods, to set aside $20\%$ of the original MNIST training set (12,000 points) for validation. On the training set, the `8' digits are subsampled such that only 5\% of them are kept. The total number of points with each digit label in the training set are $0: 4769, 1: 5382, 2: 4845, 3: 4950, 4: 4662, 5: 4300, 6: 4728, 7: 4980, 8: 234, 9: 4720$. 
The validation set is approximately balanced. Without modification, this can actually help methods that do not require group labels on the full validation set (such as ERM, \name, \textsc{George}) to more easily select models with good worst-group performance, because the rare `8' group is overrepresented in the validation set compared to the training set, so it is easier to detect poor performance on that group even when looking at overall performance or performance using noisy group labels. Thus, when computing an average metric (loss or accuracy) on a subset of the validation data, we compute a weighted average which is the sum over each group $g$ of: the average of that metric for all points in the subset whose true group label is $g$, times the proportion of group $g$ in the training dataset. This reweighting procedure is the same as what is done in \citep{sagawa2019distributionally, sohoni2020no}.

We use a 4 layer LeNet \citep{lecun1998} and the Adam optimizer for all methods. For U-MNIST, for training the robust model we train for 100 epochs with a batch size of 128, the Adam optimizer \citep{adam}, and decay the learning rate by a factor of 0.1 at epochs 50 and 75. These hyperparameters were taken from \citep{sohoni2020no}. We tune all methods over the cross product of learning rates [2e-3, 2e-4] and weight decays [1e-4, 3e-4, 1e-5]. For methods using GDRO for the second stage, we also tune the GDRO group adjustment parameter in the set \{0, 3\}, and use uniform per-group sampling, as described in \citep{sagawa2019distributionally}.

\paragraph{U-CIFAR10.} This task is based on the publicly available CIFAR-10 dataset \citep{krizhevsky2009learning} (license unknown). The U-CIFAR10 task is to classify images as ``vehicle'' or ``animal''; the groups are the original CIFAR-10 classes. We use a fixed training-validation split for all methods, to set aside $20\%$ of the original MNIST training set (10,000 points) for validation. On the training set, the `airplane' images are subsampled such that only 5\% of them are kept. The total number of points with each group label in the training set are airplane: 204, automobile: 4004, bird: 3976, cat: 4017, deer: 3997, dog: 3999, frog: 4000, horse: 3976, ship: 3957, truck: 4003.
The validation set is approximately balanced, so as described for U-MNIST we compute reweighted metrics where appropriate.

We use a ResNet-50 \citep{resnet} model and train for 200 epochs with a batch size of 128, SGD with momentum 0.9, and a cosine learning rate schedule. These hyperparameters were taken from the implementation at \url{https://github.com/kuangliu/pytorch-cifar}. We tune all methods over the cross product of learning rates [1e-2, 1e-3] and weight decays [1e-3, 3e-3, 1e-2, 3e-2, 1e-1]. For methods using GDRO for the second stage, we also tune the GDRO group adjustment parameter in the set \{0, 3\}, and use uniform per-group sampling, as described in \citep{sagawa2019distributionally}.

\paragraph{Waterbirds.} 
The Waterbirds dataset was created by \cite{sagawa2019distributionally} as a modification of the CUB dataset \cite{WahCUB_200_2011} (license unknown). It consists of different bird species (with class labels either ``waterbird'' or ``landbird'') on either a land or water background. There are 3498 training images of landbirds on land, 184 of landbirds on water, 56 of waterbirds on land, and 1057 of landbirds on land (these are the four groups).
The validation set is more balanced, so as described for U-MNIST we compute reweighted metrics where appropriate.

We use a ResNet-50 \citep{resnet} model and train for 300 epochs with a batch size of 128 and SGD with momentum 0.9. These hyperparameters were taken from \citep{sagawa2019distributionally}. We tune all methods over the (learning rate, weight decay) pairs (1e-4, 1e-1), (1e-3, 1e-4), and (1e-5, 1.0), as done in \citep{sagawa2019distributionally,liu2021just}. For methods using GDRO for the second stage, we set the GDRO group adjustment parameter to 2 as in \citep{sagawa2019distributionally}, and use uniform per-group sampling.

\paragraph{CelebA.} 
The CelebA dataset \citep{liu2015faceattributes} (license unknown) is a dataset of celebrity faces annotated with several descriptors (such as gender, hair color, wearing glasses). It is often used as a benchmark for robustness to spurious correlations. The task we consider is classifying the hair color of the person in the image as blond or non-blond, as in prior works such as \citep{sagawa2019distributionally,sohoni2020no,liu2021just}. In this dataset, hair color is spuriously correlated with gender: there are 71629 images in the ``female, non-blond'' group, 66874 in the ``female, blond'', 22880 ``male, non-blond'', and just 1387 ``male, blond''.

We use a ResNet-50 \citep{resnet} model and train for 50 epochs with a batch size of 128 and SGD with momentum 0.9. These hyperparameters were taken from \citep{sagawa2019distributionally}. We tune all methods over the (learning rate, weight decay) pairs (1e-4, 1e-2), (1e-4, 1e-4), and (1e-5, 0.1), as done in \citep{sagawa2019distributionally,liu2021just}. For methods using GDRO for the second stage, we set the GDRO group adjustment parameter to 3 as in \citep{sagawa2019distributionally}, and use uniform per-group sampling.

\subsubsection{Baseline details}
We reimplemented ERM and GDRO (and subset-GDRO) ourselves, along with \name. For the other baseline methods (George \citep{sohoni2020no}, JTT \citep{liu2021just}, and EIIL \citep{creager2021environment}), we use the authors' publicly available repositories, adapting the code (such as to plug in our dataloaders) where necessary.

All three methods require first training an ERM model. For \textsc{George}, this model is either a standard ERM model or one trained with high regularization, and is selected based on the Silhouette score of the clustered activations. For U-CIFAR10 (the only dataset which George did not evaluate on originally), we used a learning rate of 1e-3 for this model and tuned the weight decay in [1e-3, 1e-1] based on this Silhouette score criterion. For JTT and EIIL, the model is an ERM model trained with high regularization for a shorter number of epochs. For U-MNIST and U-CIFAR10 (which the JTT paper did not evaluate on originally) we used a learning rate of 2e-3 and tuned the weight decay in [1e-1, 1e-3] for this ERM model, and tuned the number of training epochs in \{1, 50\} (the tuning criterion in this case was the validation worst-group accuracy of the final robust model), for both EIIL and JTT. EIIL also did not evaluate on CelebA, so we use the same ERM model as in JTT.

For JTT and EIIL, we assume the group labels on the entire validation set are known, and use these for model selection (as done in the respective papers). This gives them a slight advantage for model selection compared to \name, in which we only use a small number of group-labeled validation examples (the same number as we use for training).
On the other hand, George does not assume validation set group labels, but rather estimates them the same way the training group labels are estiated (via clustering the activations of the last layer).

\subsubsection{\name{} details}
For \name-Base, we train a supervised group classifier model (with GDRO) to predict group pseudolabels, as described in Section \ref{sec:method}. This is challenging due to the low number of points with known group labels: for instance, with 8 group-labeled points per group on Waterbirds or CelebA, the entire training dataset for this stage is 32 examples. Thus, we generally need to train for more epochs to obtain reasonable results. We also use mild data augmentation (random crops and flips) for the Stage 1 group classifier only to help deal with this lack of data (although, similarly to prior works, we do not use data augmentation in Stage 2, except for the experiment with RotNet with data augmentation).

For the learning rate and weight decay for the group classifier, we tune over the same pairs of hyperparameters as described above. For the number of epochs, we train for $500 / (n_{lab} / 64)$ epochs, where $n_{lab}$ is the number of group-labeled examples in the smallest group (in our experiments, we evaluate the settings $n_{lab} \in \{8, 16, 32, 64\}$). (We evaluate the validation accuracy every $n_{lab} / 64$ epochs to keep a constant number of validation evaluations.) We set the batch size for training the group classifier to be the minimum of 128 or the total number of group-labeled examples.

We note that the group classifier and end model need not have the same architecture. For instance, one could potentially select a smaller model for the group classifier since it is trained with a small amount of data. In our preliminary experiments, however, we found that using the larger ResNet-50 model (the same architecture as the Stage 2 model) performed better than using a smaller ResNet-18 on all datasets.

For $\name$ (and Subset-GDRO), the points for which we know the group label are selected randomly from the training and validation datasets. We select different sets of these points depending on the random seed (but for a fixed seed, these sets are the same, to facilitate direct comparisons). We do this to avoid over-indexing interpretation of results to a particularly ``easy'' or ``difficult'' set chosen by happenstance. In the real world, of course, the points for which the group labels are known would generally be a fixed set. We note that hyperparameters should not be tuned by looking at averaged metrics over the different seed results, because in a way this ``shares information'' between different trials with different group-labeled points, so in some sense is using more group label information than it should. Thus, we indeed select the best hyperparameters separately for each seed (based on the appropriate validation metric for that seed). For consistency, we do this for \emph{all} methods (e.g., tune hyperparameters on a per-seed basis based on the validation metric).

\subsection{Additional Results}
\subsubsection{Worst-Group Performance}
In Table~\ref{tab:wg-results-app} we provide a more complete version of Table~\ref{tab:wg-results}.

In this table, all results are rerun by us except for \textsc{George} on U-MNIST, CelebA and Waterbirds (since the original \textsc{George} paper already reports results over 5 seeds on those datasets). We ran 5 seeds for all methods except 10 seeds for \name, due to the increased variability from selecting different subsets of points with known group labels. We note that our averaged results were somewhat better on Waterbirds and worse on CelebA than those reported in the \textsc{Jtt} paper (which reports results from one trial). Similarly, our results for EIIL on Waterbirds are also somewhat better than those reported in the EIIL paper \citep{creager2021environment}.

Separately, the recent work \citep{zhang2022correctncontrast} proposes a method (CnC) based on contrastive learning for improving robustness to spurious correlations. This method does very well on the spurious correlation datasets (Waterbirds and CelebA), even exceeding the performance of \name{} on CelebA. However, we found that even after hyperparameter tuning, CnC did not work well on U-MNIST and U-CIFAR10 (attaining worst-group accuracies lower than those of ERM), which are tasks without explicit spurious correlations (even though CnC can in principle be applied to such tasks).

\begin{table*}[t]
\caption{Extended version of Table~\ref{tab:wg-results}, with additional baseline results (EIIL \citep{creager2021environment}, Subset-GDRO) and results for \name~when there are 8, 16, 32, and 64 group-labeled examples provided from each group (in each of the training and validation sets). As Waterbirds only has 56 training points in the smallest group, the last setting (64) does not apply to it.
}
\label{tab:wg-results-app}
\centering
{\small
{
\resizebox{\textwidth}{!}{\begin{tabular}{@{}lccccccccc@{}}
\toprule
Method              & \multicolumn{2}{c}{U-MNIST} & \multicolumn{2}{c}{Waterbirds} & \multicolumn{2}{c}{CelebA} & \multicolumn{2}{c}{U-CIFAR10} \\
Accuracy ($\%$)     & Worst-group       & Avg.    & Worst-group     & Avg.  & Worst-group & Avg.  & Worst-group           & Avg.         \\ \midrule
ERM              &  $93.4 \pm 0.5$  &  $99.2 \pm 0.0$  &  $60.6 \pm 3.3 $  &  $97.3 \pm 0.1$ &  $39.7 \pm 3.0$  &  $95.7 \pm 0.1$  &  $88.4 \pm 1.4$  &  $99.5 \pm 0.1$  \\
EIIL &  $97.2 \pm 0.5$  &  $98.9 \pm 0.2$  &  $87.3 \pm 4.2$  &  $93.1 \pm 0.6$  &  $81.3 \pm 1.4$  &  $89.5 \pm 0.4$ &  $85.3 \pm 1.4$  &  $99.4 \pm 0.1$  \\
\textsc{George}  & $95.7 \pm 0.6$ &     $97.9 \pm 0.2$  &  $76.2 \pm 2.0$  &  $95.7 \pm 0.5$  &  $53.7 \pm 1.3$ &  $94.6 \pm 0.2$  &  ${93.4} \pm 5.8$  &  $98.9 \pm 0.3$  \\
\textsc{Jtt}     &  ${96.2} \pm 0.7$ &  $98.4 \pm 0.4$  &  $88.0 \pm 0.7$  & $91.7 \pm 0.8$  & $77.8 \pm 2.0$ & $87.2 \pm 1.2$  &  $89.0 \pm 4.7$  &  $94.6 \pm 1.3$  \\
\midrule
Subset-GDRO (8)     &  $66.9 \pm 5.3$          &  $84.0 \pm 1.8$  &  $76.4 \pm 3.6$ &  $81.0 \pm 4.9$  &  $56.9 \pm 13.1$ &  $74.2 \pm 5.3$ & $83.4 \pm 3.9$ &  $93.5 \pm 0.6$  \\
Subset-GDRO (16)     &  $77.4 \pm 2.5$          &  $89.2 \pm 1.1$  &  $83.9 \pm 1.4$ &  $86.1 \pm 1.7$  &  $75.5 \pm 5.9$ &  $81.9 \pm 1.4$ & $86.0 \pm 2.2$ &  $93.3 \pm 1.8$  \\
Subset-GDRO (32)     &  $85.4 \pm 1.4$          &  $92.4 \pm 0.4$  &  $86.9 \pm 1.0$ &  $88.6 \pm 0.5$  &  $76.6 \pm 4.4$ &  $85.5 \pm 1.8$ & $88.6 \pm 2.4$ &  $95.2 \pm 0.9$  \\
Subset-GDRO (64)     &  $89.6\pm 1.8$          &  $94.7 \pm 0.7$  &  - &  -  &  $79.0 \pm 4.4$ &  $87.3 \pm 1.2$ & $91.7 \pm 1.5$ &  $96.6 \pm 0.4$  \\
\midrule
\name-Base (8)  &  $96.2 \pm 0.8$ &  $99.2 \pm 0.1$   &  $83.0 \pm 5.9$  &  $94.4 \pm 2.8$  & $81.1 \pm 3.2$  &  $92.9 \pm 0.3$ &  $90.4  \pm 1.9$ &  $99.2 \pm 0.2$ \\ 
\name-Base (16)  &  $96.4\pm 1.0$ &  $99.1 \pm 0.2$   &  $86.9 \pm 2.3$  &  $94.4 \pm 2.9$  & $83.0 \pm 4.1$  &  $92.9 \pm 0.8$ &  $90.5 \pm 3.3$ &  $99.1 \pm 0.4$ \\ 
\name-Base (32)  &  ${96.9} \pm 0.9$ &  $99.1 \pm 0.3$   &  ${89.6} \pm 0.9$  &  $94.3 \pm 1.3$  & ${83.8} \pm 2.7$  &  $92.8 \pm 0.6$ &  ${94.5} \pm 1.1$ &  $98.9 \pm 0.3$ \\ 
\name-Base (64)  &  $97.4 \pm 0.8$ &  $99.0 \pm 0.3$   &  -  &  - & $84.3 \pm 2.0$  &  $92.8 \pm 0.5$ &  $95.2 \pm 0.8$ &  $98.9 \pm 0.4$ \\ 
\midrule
Full-GDRO             &  $98.6 \pm 0.2$  &  $99.1 \pm 0.1$ &  $90.9 \pm 0.2$ &  $92.8 \pm 0.2$  &  $89.3 \pm 0.9$  & $92.8 \pm 0.1$  &  $97.0 \pm 0.3$  &  $99.2 \pm 0.3$ \\
\end{tabular}}
}}
\end{table*}

\subsubsection{Group Prediction Accuracy}
In Figure~\ref{fig:group-results} we provide more results on the performance of the ``Stage 1'' group classification model. We plot both worst-group and average accuracies for the group classifier, corresponding to the same settings as in Table~\ref{tab:wg-results-app}. 
\begin{figure*}
    \centering
    \includegraphics[width=0.45\textwidth]{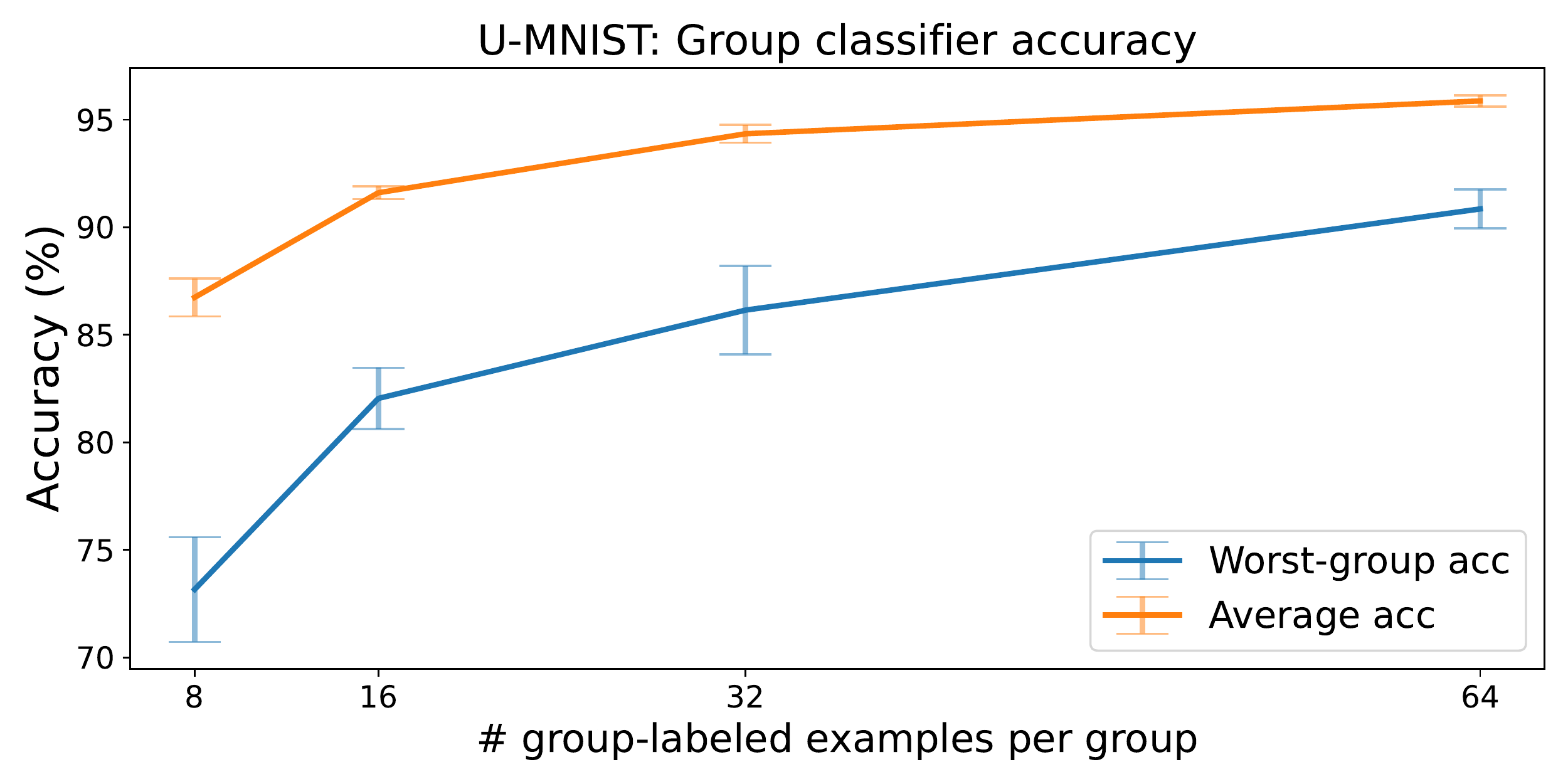}\quad\includegraphics[width=0.45\textwidth]{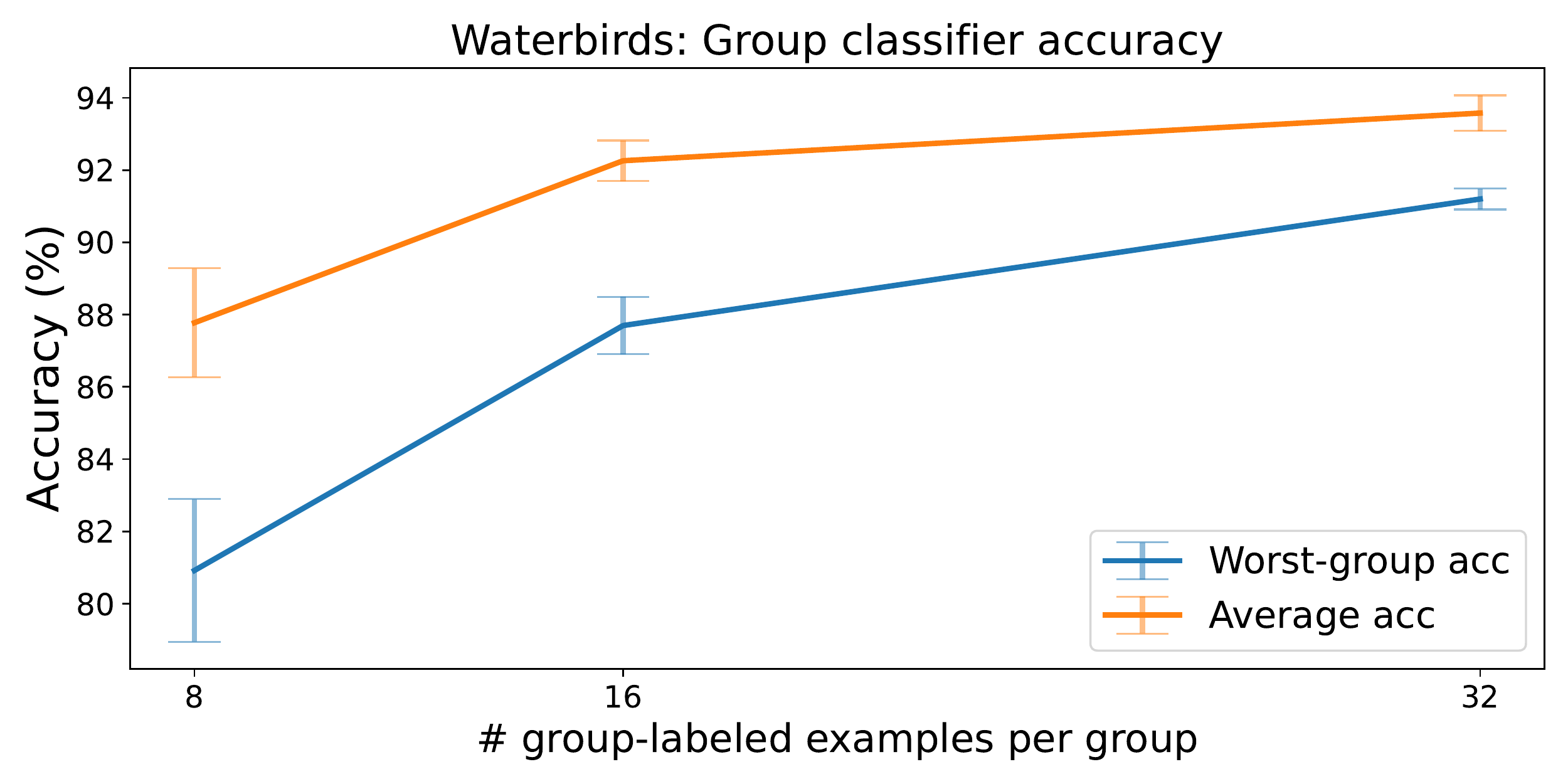}
    \\
    \includegraphics[width=0.45\textwidth]{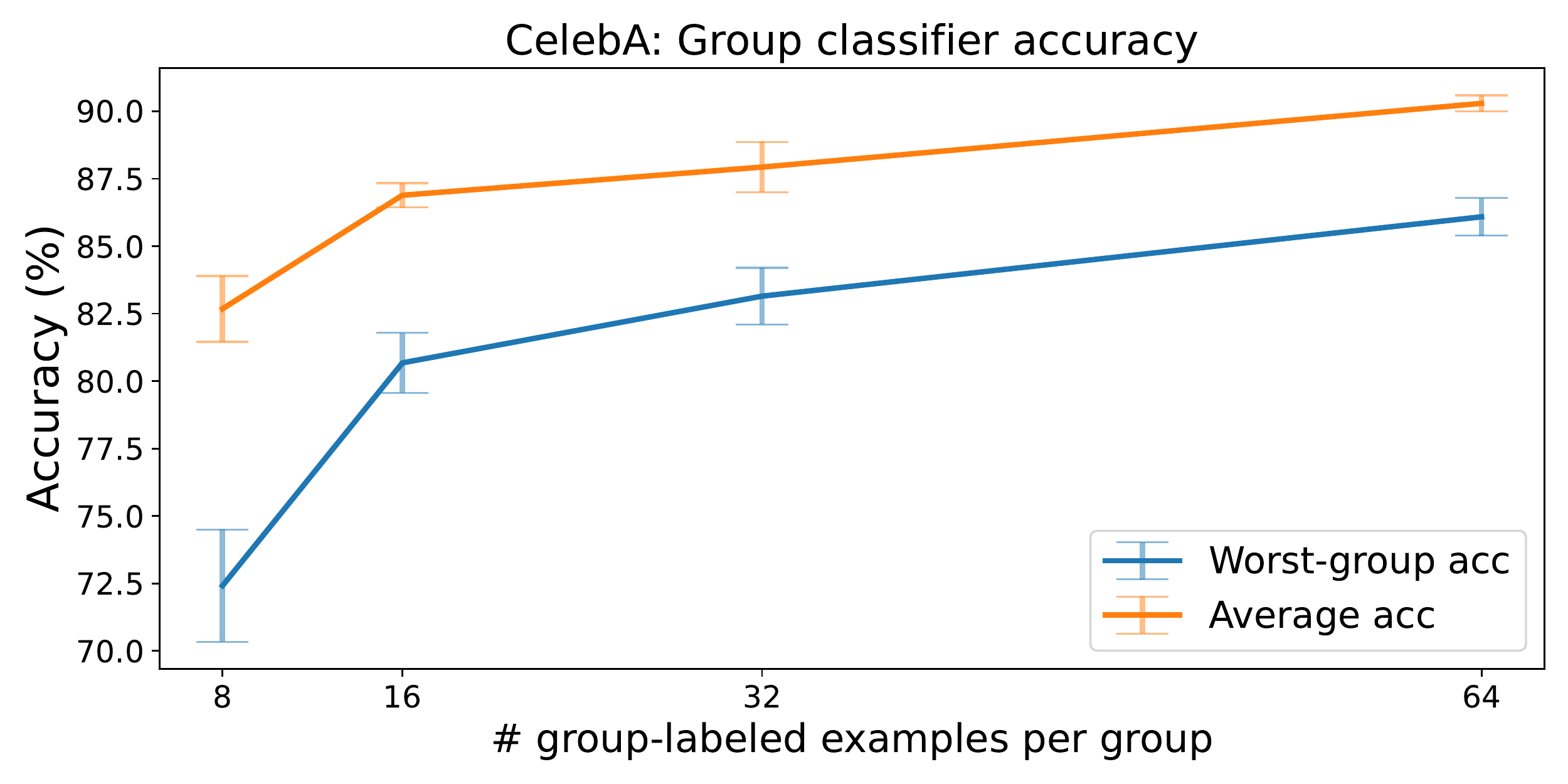}\quad\includegraphics[width=0.45\textwidth]{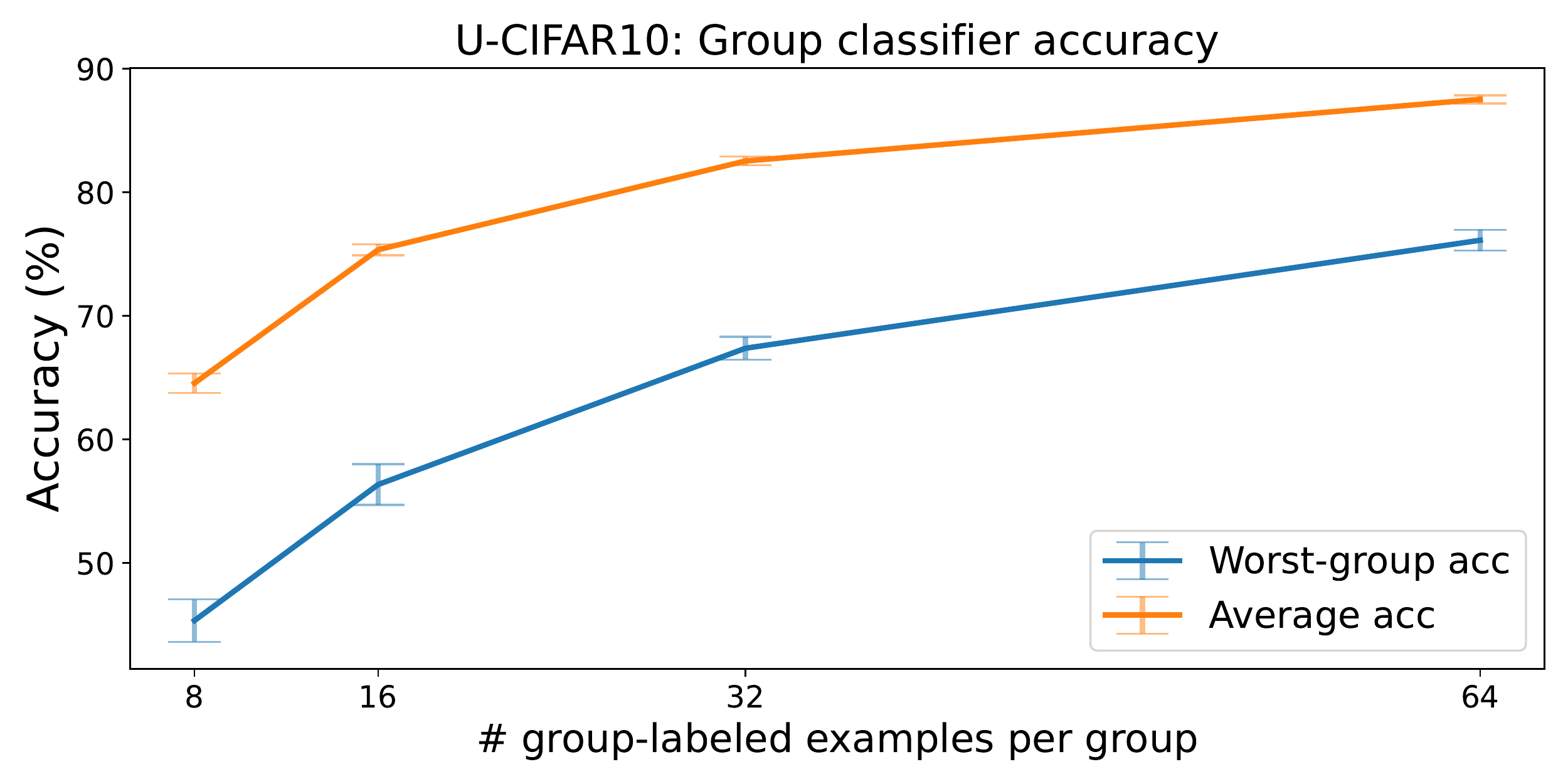}
    \caption{Worst-group and average accuracy of the \emph{group} classifier.}
    \label{fig:group-results}
\end{figure*}

\subsubsection{\name-SSL Details}
For \name-SSL, we use FixMatch \citep{sohn2020fixmatch} to train the semi-supervised group classifier. We adapt the PyTorch implementation at \url{https://github.com/kekmodel/FixMatch-pytorch}, modifying it to use the class label the same way as described in Section~\ref{sec:method} in order to assign zero probability to the groups that do not belong to the given class. Other than that, we use the default FixMatch hyperparameters. The group classifier used is simply the model at the end of FixMatch training (so we do not use the validation set at all for Stage 1). For the Stage 2 GDRO model, we used the same hyperparameter search and model selection approach as described in Appendix \ref{app:train-details}.

\subsubsection{Additional Ablation Results}
In this section we include additional details and results that were omitted from Sections \ref{sec:random-flip} to \ref{sec:pretraining}. (Note: We average over 5 seeds for all ablation results.)

\paragraph{Random Flipping Experiment} As described in Section \ref{sec:random-flip}, for this experiment we created random ``synthetic group pseudolabels'' to have the same confusion matrix with respect to the true group labels as the confusion matrix of the \name~group pseudolabels with respect to the true group labels. Specifically, for each setting (dataset, seed, and number of known group labels) we computed the confusion matrix of the group predictions output by the corresponding \name{} model (i.e., the one selected as the ``best model'' based on the criteria described in Appendix \ref{app:train-details}), and then created the ``synthetic pseudolabels'' by taking the true group labels and flipping randomly chosen ones to match the confusion matrix. We then used these synthetic pseudolabels in the GDRO objective, with the same hyperparameters as those of the \name{} model. (Thus, one possible explanation for the fact that the \name{} models generally outperform the ``randomly flipped'' models slightly is the fact that we did not perform a separate hyperparameter search for the ``randomly flipped'' model, instead using the same ones as those of the \name{} model with the same confusion matrix.)

\paragraph{Pretrained Model Choice} 
In this section, we provide additional results on CelebA, Waterbirds, and U-CIFAR10 comparing the use of pretrained models---specifically, comparing a model trained on the supervised ImageNet task vs. the RotNet model \citep{gidaris2018unsupervised} from the VISSL library \citep{goyal2021vissl}---as well as exploring the use of data augmentation for the latter model to boost its performance closer to that of the supervised pretrained model. Results for ERM and GDRO are in Table~\ref{tab:rotnet}; we plot the worst-group accuracies for \name{} in Figure~\ref{fig:rotnet}. We observe that data augmentation does not seem to benefit the RotNet model much on U-CIFAR10, but it does on the other tasks (on CelebA, augmentation even boosts the performance using RotNet past the performance using the supervised ImageNet model, although of course the performance of the latter could likely also be boosted by using data augmentation). For the RotNet training with augmentation, The augmentations used in both Stage 1 and Stage 2 of \name{} are random crops, random flips, and random rotations (of up to 15 degrees); the reason we added the random rotations is because rotation is a key part of pretraining the RotNet model itself.

\begin{figure*}
    \centering
    \includegraphics[width=0.45\textwidth]{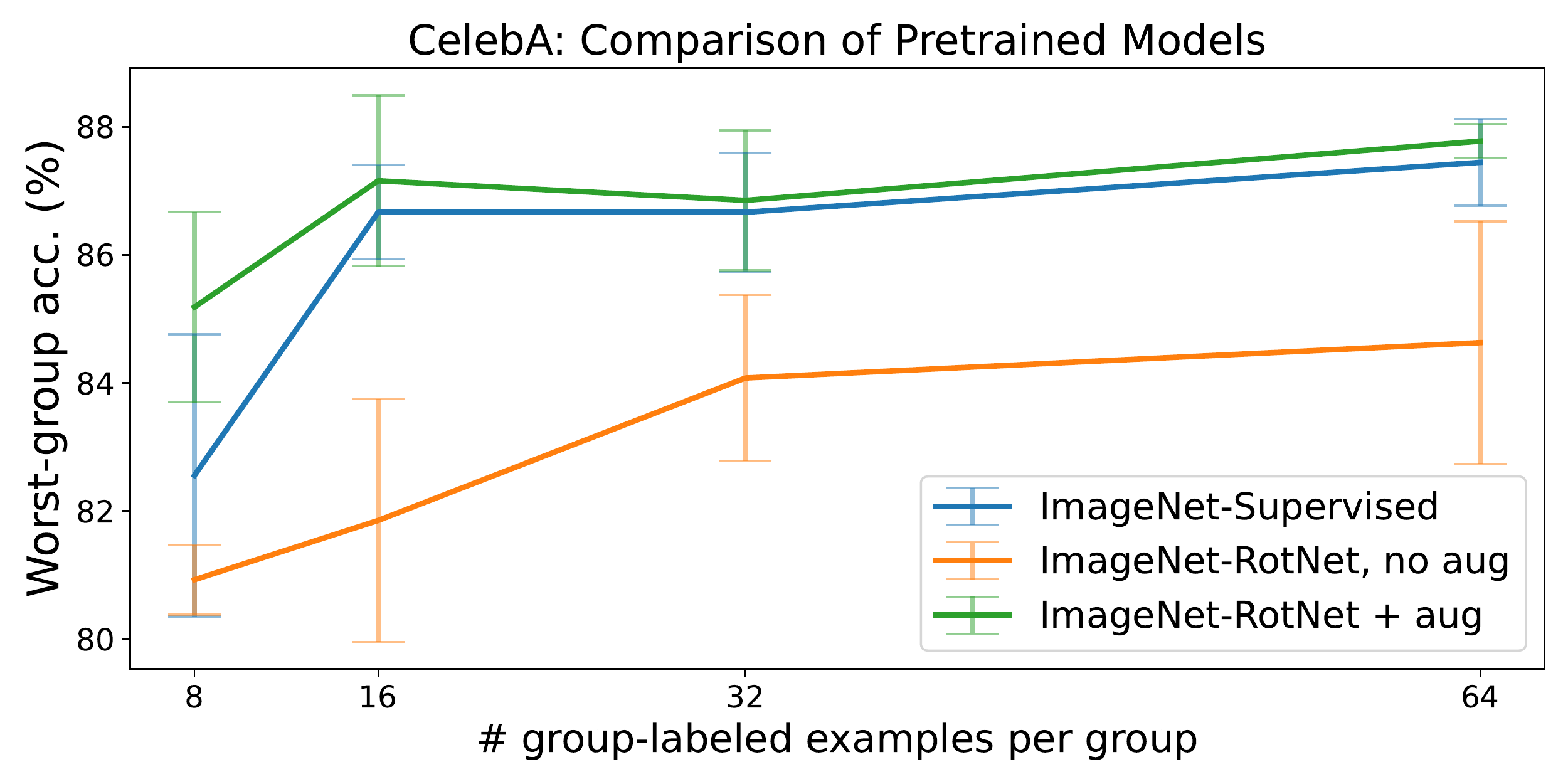}\quad\includegraphics[width=0.45\textwidth]{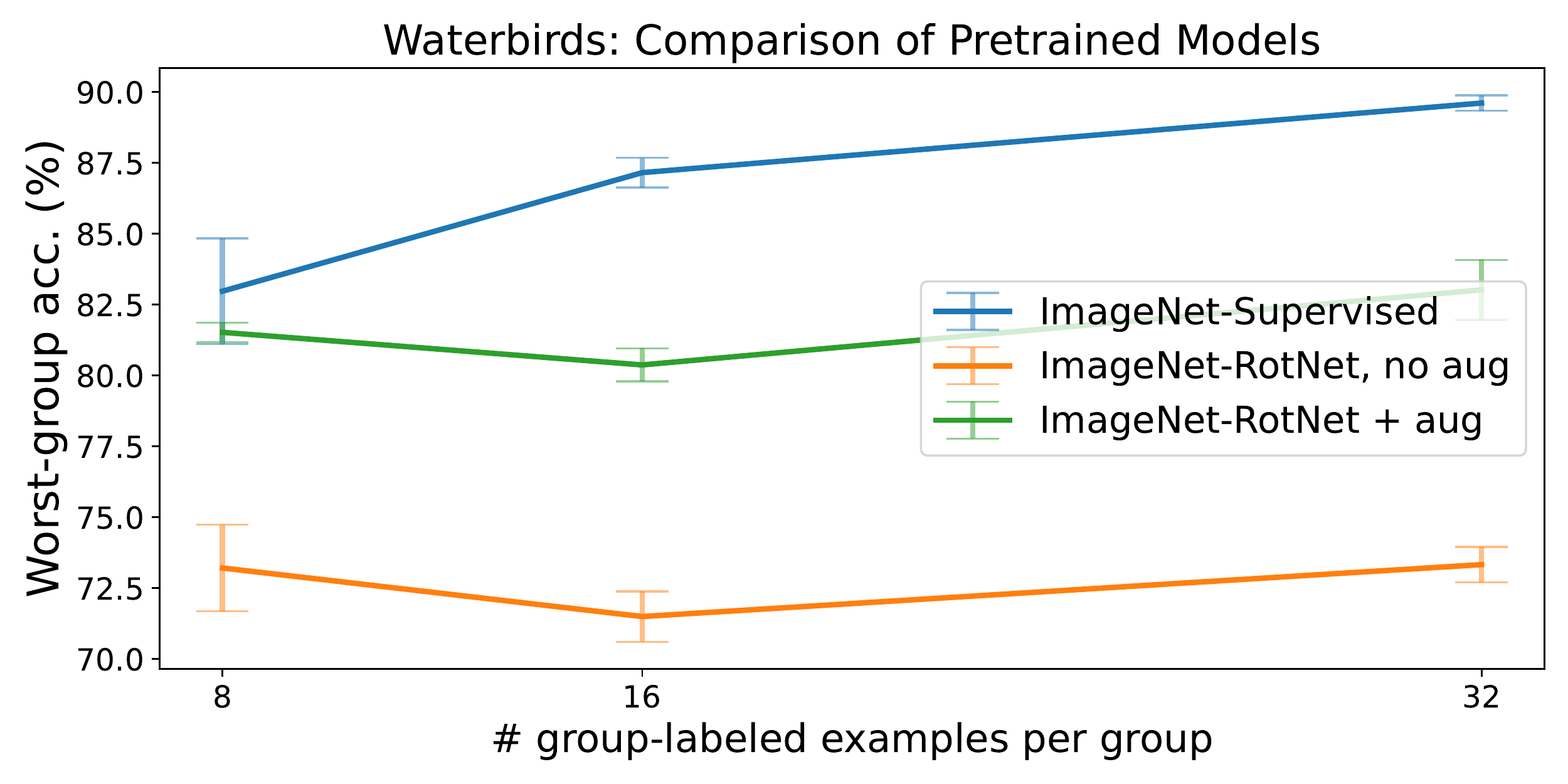}\quad\includegraphics[width=0.45\textwidth]{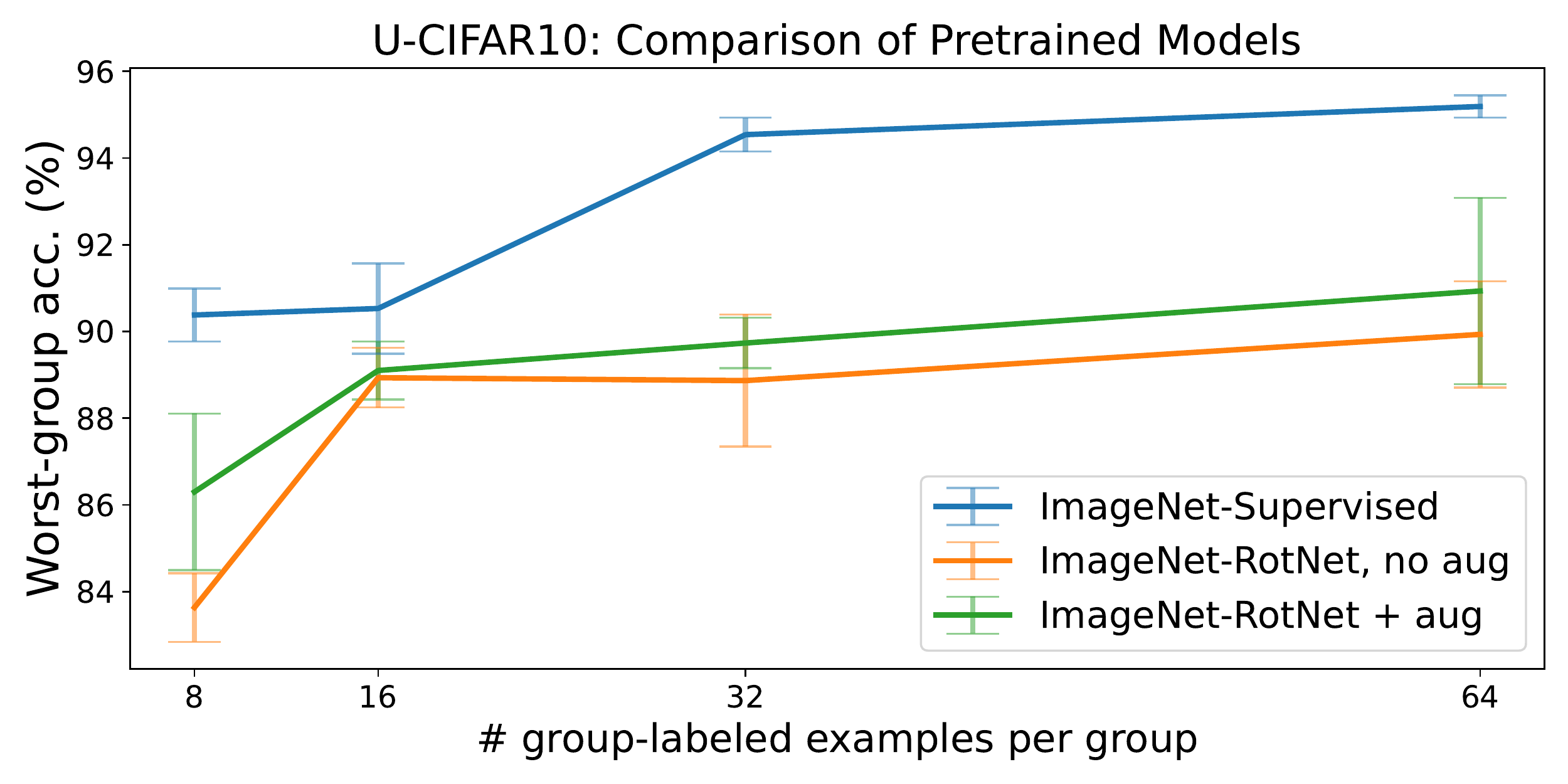}
    \caption{Worst-group accuracy using different initial pretrained models.}
    \label{fig:rotnet}
\end{figure*}

\end{document}